%% file: neurips_main.tex
\documentclass{article}

\PassOptionsToPackage{round,authoryear}{natbib}

\usepackage[main,preprint]{neurips_2026}
\usepackage[utf8]{inputenc}
\usepackage[T1]{fontenc}
\usepackage{hyperref}
\usepackage{url}
\usepackage{microtype}
\usepackage{xcolor}
\usepackage{graphicx}
\usepackage{subcaption}
\usepackage{booktabs}
\usepackage{adjustbox}
\usepackage{multirow}
\usepackage{array}
\usepackage{amsmath,amsfonts,amssymb,mathtools}
\usepackage{amsthm}
\usepackage{algorithm}
\usepackage{algorithmic}
\usepackage[capitalize,noabbrev]{cleveref}

\newtheorem{assumption}{Assumption}

\theoremstyle{plain}
\newtheorem{theorem}{Theorem}[section]

\newtheorem{lemma}[theorem]{Lemma}
\newtheorem{corollary}[theorem]{Corollary}

\theoremstyle{definition}
\newtheorem{definition}[theorem]{Definition}
\theoremstyle{remark}
\newtheorem{remark}[theorem]{Remark}

\newcommand{\R}{\mathbb{R}}
\newcommand{\E}{\mathbb{E}}

\newcommand{\iid}{\overset{\mathrm{iid}}{\sim}}
\newcommand{\prox}{\operatorname{prox}}
\newcommand{\dist}{\operatorname{dist}}

\newcommand{\proj}{\mathrm{proj}}

\newcommand{\maybeincludegraphics}[2][]{%
  \IfFileExists{#2}{%
    \includegraphics[#1]{#2}%
  }{%
    \fbox{%
      \begin{minipage}[c][0.23\textheight][c]{0.92\linewidth}
      \centering
      Missing figure file: \texttt{\detokenize{#2}}\\[0.5em]
      Upload this file to Overleaf or replace the filename.
      \end{minipage}%
    }%
  }%
}

\providecommand{\answerYes}{\textcolor{blue}{[Yes]}}
\providecommand{\answerNo}{\textcolor{orange}{[No]}}
\providecommand{\answerNA}{\textcolor{gray}{[N/A]}}

\title{Oracle-Robust Online Alignment for Large Language Models}
\author{%
  Zimeng Li\\
  Purdue University
  \And
  Mudit Gaur\\
  The Ohio State University
  \And
  Vaneet Aggarwal\\
  Purdue University
}

\begin{document}

\maketitle
\begin{abstract}
We study online alignment of large language models under misspecified preference feedback, where the observed preference oracle deviates from an ideal but unknown ground-truth oracle. The online LLM alignment problem is a bi-level reinforcement problem due to the coupling between data collection and policy updates. Recently, the problem has been reduced to a tractable single-level objective in the SAIL (Self-Improving Efficient Online Alignment) framework. In this paper, we introduce a pointwise oracle uncertainty set in this problem and formulate an oracle-robust online alignment objective as a worst-case optimization problem. For log-linear policies, we show that this robust objective admits an exact closed-form decomposition into the original loss function plus an explicit sensitivity penalty. In the finite-response log-linear setting, we develop projected stochastic composite updates for the resulting weakly convex objective and prove $\widetilde{O}(\varepsilon^{-2})$ oracle complexity for reaching approximate stationarity.
\end{abstract}

\input{paper_body}

\begin{ack}
This work was supported in part by the National Science Foundation (NSF) Future Manufacturing (FM) program through Award No. 2229250. The authors declare no competing interests.
\end{ack}
\bibliographystyle{plainnat}
\bibliography{references}

\appendix
\input{appendixA}
\input{appendixB}
\input{appendixC}
\input{appendixD}

\newpage
\input{checklist}

\end{document}

%% file: paper_body.tex
\section{Introduction}
Reinforcement learning from human feedback (RLHF) aligns large language models by updating a policy $\pi_\theta$ from pairwise preferences produced by a nominal training oracle $P^\star$, which may be a pool of human annotators or a learned reward model \citep{christiano2017deep,ouyang2022training,ziegler2019fine,stiennon2020learning,bai2022training}. This available oracle may be misspecified relative to an unobserved ideal oracle $P^{\mathrm{id}}$: annotators disagree, labeling standards drift, and style confounders such as verbosity can systematically bias judgments \citep{poddar2024personalizing,ramesh2024grpo}. In online RLHF, where the policy controls the distribution of sampled comparisons, misspecification of the nominal oracle can be amplified by the feedback loop and lead to over-optimization of oracle quirks rather than the intended preferences \citep{skalse2022rewardhacking,gao2023scaling}.

Robustness to oracle misspecification therefore needs to be modeled through an uncertainty set and a worst-case objective, rather than only through i.i.d.\ label noise. Two existing lines of work each handle part of this problem. Online RLHF methods such as SAIL \citep{ding2024sail} account for policy-induced sampling through a bilevel formulation, but assume the oracle matches the model and provide no distributionally robust guarantee. Distributionally robust DPO variants \citep{xu2025robustdpo,wu2025drdpo} optimize a worst-case loss over an uncertainty set, but only on a fixed offline dataset, sidestepping the policy-dependent comparison distribution. The missing ingredient, which we address, is robustness against an oracle uncertainty set under on-policy sampling.

At a glance, plugging an oracle uncertainty set into the SAIL bilevel program may look like a routine application of distributionally robust optimization. It is not. The worst-case objective places a supremum over adversarial oracles $P\in U^W(P^\star,\rho)$ inside an expectation under the policy-induced law $d_\theta$, which is generally intractable in bilevel form. Our first technical contribution is to show that for log-linear policies,  this inner supremum admits an exact closed-form decomposition,
\[
  L^W_\rho(\theta)=L^{\mathrm{SAIL}}(\theta)+\lambda R(\theta),
  \qquad \lambda=\rho\beta,
\]
where $R(\theta)$ is the expected absolute pairwise log-ratio score under $d_\theta$. This decomposition turns a robust bilevel problem into a single-level penalized objective, but introduces a new difficulty: $R(\theta)$ is the expectation of a pointwise convex quantity under a $\theta$-dependent sampling law, and this coupling destroys convexity. Our second contribution resolves this for log-linear policies over a finite response set, by proving that $R(\theta)$ is weakly convex with an explicit constant and by analyzing a projected stochastic composite method using the Moreau envelope as a smooth stationarity surrogate.

Our contributions are as follows.
\begin{itemize}
    \item We formulate oracle-robust online alignment by combining SAIL with a pointwise oracle uncertainty set $U^W(P^\star,\rho)$, leading to the worst-case objective $L^{W}_{\rho}(\theta)$ under policy-induced sampling.
    \item Under log-linear policies, we give an exact decomposition $L^W_\rho(\theta)=L^{\mathrm{SAIL}}(\theta)+\lambda R(\theta)$ and identify $R(\theta)$ as a sensitivity penalty given by an expected absolute pairwise score.
    \item We prove weak convexity of the constrained robust objective, analyze projected stochastic composite updates, and establish an $\widetilde O(\varepsilon^{-2})$ oracle complexity for squared Moreau-envelope stationarity. Setting $\rho=0$ gives a convergence-to-stationarity guarantee for nominal SAIL.
\end{itemize}

\section{Problem Setup}
\label{sec:problem_setup}

Let $X$ be the prompt (context) space and $Y$ the response space. For each parameter vector $\theta \in \Theta \subseteq \R^d$, the language model induces a conditional distribution (policy) $\pi_\theta(\cdot \mid x)\in \Delta(Y)$ for each $x\in X$, where $\Delta(Y)$ denotes the probability simplex over $Y$. We view each $y\in Y$ as a finite token sequence.

We model preferences over response pairs by a (possibly unknown) preference oracle
\(
\mathcal{O}
:=\Bigl\{P
\Bigm|\ P(y_1\succ y_2\mid x)=1-P(y_2\succ y_1\mid x)\Bigr\}.
\)
For $z=(x,y_1,y_2)\in Z:=X\times Y\times Y$, we write $P(1\mid z):=P(y_1\succ y_2\mid x)$ and interpret the (stochastic) preference label $\mathbb{I}\in\{0,1\}$ as
\(
\mathbb{I} \mid z \sim \mathrm{Bernoulli} \! \bigl(P(1\mid z)\bigr),\)
 hence $\mathbb{I}=1 \Leftrightarrow y_1 \succ y_2$.
The sampling procedure is policy-dependent: draw a prompt $x\sim \mu$ from a fixed distribution $\mu$ over $X$, then sample two responses independently from $\pi_\theta(\cdot\mid x)$,
$
\;y_1,y_2 \iid \pi_\theta(\cdot\mid x).
$ 
This induces the data-generating distribution on $Z$,
\begin{equation}
d_\theta(x,y_1,y_2):=\mu(x)\,\pi_\theta(y_1\mid x)\,\pi_\theta(y_2\mid x).
\label{eq:data_gen}
\end{equation}

Let $P^\star\in\mathcal{O}$ denote the nominal training preference oracle. Fix $\rho>0$. For each $z=(x,y_1,y_2)\in Z$, define the pointwise uncertainty set
\begin{equation}
U^W_z(P^\star,\rho)
:=\Bigl\{P\in\mathcal{O}\ \Bigm|\ \bigl|P(1\mid z)-P^\star(1\mid z)\bigr|\le \rho\Bigr\}.
\label{eq:pointwise_U}
\end{equation}

Equivalently, for fixed $z$, view $P(\cdot|z)$ and $P^\star(\cdot|z)$
as Bernoulli distributions over the binary preference-label space
$\{0,1\}$. Under the ground cost
$c(i,i')=|i-i'|$ for labels $i,i'\in\{0,1\}$, the 1-Wasserstein distance \citep{peyre2019computational}
satisfies
$$
W_1\!\left(\mathrm{Ber}(P(1|z)),\mathrm{Ber}(P^\star(1|z))\right)
=
|P(1|z)-P^\star(1|z)|.
$$
Thus $U_z^W(P^\star,\rho)$ is exactly a radius-$\rho$ Wasserstein ball
around the Bernoulli preference oracle at comparison triple $z$.

The corresponding global uncertainty set requires the pointwise constraint to hold uniformly over all comparison triples:
\begin{equation}
\begin{aligned}
    U^W(P^\star,\rho)&:=\bigcap_{z\in Z}U^W_z(P^\star,\rho)=\Bigl\{P\in\mathcal{O}\ \Bigm|\ \sup_{z\in Z}\bigl|P(\cdot\mid z)-P^\star(\cdot\mid z)\bigr|\le \rho\Bigr\}.
\end{aligned}
\label{eq:global_U}
\end{equation}
$U^W(P^\star,\rho)$ models an adversarial but uniformly bounded misspecification of the (conditional) preference probability across all prompts and response pairs that may be encountered under the policy-induced sampling in~\eqref{eq:data_gen}.

Fix a prompt (context) $x\in X$. The policy $\pi_\theta(\cdot\mid x)$ independently generates a pair of responses
$y_1,y_2\sim \pi_\theta(\cdot\mid x)$, which are then compared by a preference oracle. For the nominal training oracle $P^\star$, we assume it satisfies the Bradley--Terry model \citep{bradley1952rank}: for all $x\in X$ and $y_i\in Y$, there exists a latent reward function   $r^\star(x,y_i)$ such that
\begin{equation}
\label{eq:bt_model}
P^\star(y_1 \succ y_2 \mid x)
=
\sigma\!\bigl(r^\star(x,y_1)-r^\star(x,y_2)\bigr),
\end{equation}
where $\sigma(u):=(1+e^{-u})^{-1}$ denotes the logistic sigmoid, $i \in \{1,2\}$.

We impose a mild margin condition to ensure that all admissible oracles remain valid.
\begin{assumption}[Nondegenerate nominal oracle and admissible radius]\label{ass:nondegenerate_oracle}
There exists a constant $\delta\in(0,1/2]$ such that for all $(x,y_1,y_2)\in X\times Y\times Y$,
\begin{equation}
\label{eq:oracle_nondegenerate}
\delta \ \le\ P^\star(y_1 \succ y_2 \mid x)\ \le\ 1-\delta .
\end{equation}
We restrict the oracle uncertainty radius to $\rho\in(0,\delta)$, so that every preference oracle $P$ is nondegenerate, i.e.  $\forall (x,y_1,y_2),\ \forall P\in U^W(P^\star,\rho)$
\[
0\le\delta-\rho \ \le\ P(y_1 \succ y_2 \mid x)\ \le\ 1-(\delta-\rho)\le 1
.
\]
\end{assumption}

\begin{remark}
Assumption~1 bounds the nominal preference probability away from $0$ and
$1$, so that all radius-$\rho$ perturbations with $\rho<\delta$ remain
valid Bernoulli probabilities. Under the Bradley--Terry model,
\[
\delta \le \sigma(r^\star(x,y_1)-r^\star(x,y_2)) \le 1-\delta
\quad\Longleftrightarrow\quad
|r^\star(x,y_1)-r^\star(x,y_2)|
\le \log\frac{1-\delta}{\delta}.
\]
Thus the assumption rules out essentially deterministic comparisons
over the support considered by the policy-induced sampling distribution. In particular, if $|r^\star(x,y_1)-r^\star(x,y_2)|\le M$ for all
$(x,y_1,y_2)$, then the assumption holds with $\delta=\sigma(-M)$, which
gives a conservative estimate of $\delta$. The restriction $\rho<\delta$
only ensures that every interval
$[P^\star(1\mid z)-\rho,\,P^\star(1\mid z)+\rho]$ lies in $(0,1)$. 
\end{remark}


\paragraph{Background (SAIL).}
SAIL is a preference-based RLHF framework for online alignment that explicitly models the coupling between (i) learning from pairwise preference feedback and (ii) updating the policy that generates the responses being compared. In the online regime, the preference data distribution is policy-dependent, and SAIL represents this dependence via a bilevel program: a reward model $r$ is fit from Bradley--Terry comparisons on responses sampled from the KL-regularized optimal policy induced by $r$,
\begin{equation}
\label{eq:sail_bilevel}
\begin{aligned}
\text{(upper)}&\quad \min_{r}-\E_{x\sim\mu,\,y_i\sim\pi_r^\star(\cdot\mid x),(y_w\succ y_\ell)\sim P^\star}
\Bigl[\log \sigma\bigl(r(x,y_w)-r(x,y_\ell)\bigr)\Bigr] \\
\text{(lower)}&\quad \text{s.t.}\ \pi_r^\star \in \arg\max_{\pi}
\E_{x\sim\mu,\,y\sim\pi(\cdot\mid x)}\Bigl[r(x,y)-\beta D_{\mathrm{KL}}\!\bigl(\pi(\cdot\mid x)\,\|\,\pi_{\mathrm{SFT}}(\cdot\mid x)\bigr)\Bigr],
\end{aligned}
\end{equation}

Direct differentiation through the inner solution mapping $r\mapsto \pi_r^\star$ requires hypergradient computations. SAIL circumvents this by exploiting the reward--policy equivalence for KL-regularized RL: any optimizer $\pi_r^\star$ satisfies
\begin{equation}
\label{eq:sail_reward_policy_equiv}
r(x,y)=\beta \log \frac{\pi_r^\star(y\mid x)}{\pi_{\mathrm{SFT}}(y\mid x)}+\beta\log Z_r(x),
\end{equation}
for a normalization $Z_r(x)$ independent of $y$. For a policy $\pi_\theta$, define the pairwise log-ratio score
\[
h_\theta(x,y_a,y_b)
:=\log \frac{\pi_\theta(y_a\mid x)}{\pi_{\mathrm{SFT}}(y_a\mid x)}
-\log \frac{\pi_\theta(y_b\mid x)}{\pi_{\mathrm{SFT}}(y_b\mid x)}.
\]
Substituting~\eqref{eq:sail_reward_policy_equiv} into the upper-level BT likelihood reduces the bilevel program to the single-level SAIL objective
\begin{equation}
 \label{eq:Lsail_single_level}
L^{\mathrm{SAIL}}(\theta)
:= -\mathbb{E}_{z\sim d_\theta,\;(y_w,y_\ell)\sim P^\star}
\left[\log \sigma\!\left(\beta h_\theta(x,y_w,y_\ell)\right)\right].
\end{equation}
For any oracle $P$, let $L^{\mathrm{SAIL}}_P(\theta)$ denote the same objective with preference outcomes generated by $P$ instead of $P^\star$. We define the oracle-robust alignment objective as the worst-case SAIL value over adversarial oracles in the uncertainty set centered at the nominal oracle $P^\star$:
\begin{equation}
\label{eq:Lrobust_def}
L^W_{\rho}(\theta)
:= \sup_{P \in U^W(P^\star,\rho)} L^{\mathrm{SAIL}}_P(\theta).
\end{equation}

\section{Proposed Approach}
\label{sec:main}
In this section, we first relate our oracle-robust objective $L^W_{\rho}$ to the nominal SAIL objective
$L^{\mathrm{SAIL}}$ by showing that $L^W_{\rho}$ admits an \emph{exact} decomposition into the SAIL loss
and an explicit regularization term. We then study the regularity of the two components separately:
we impose a smoothness condition on $L^{\mathrm{SAIL}}$ and analyze the policy-dependent penalty $R(\theta)$.
Under mild assumptions, these results together imply that $L^W_{\rho}$ is weakly convex.


\subsection{Decomposition of
\texorpdfstring{$L^{W}_\rho(\theta)$}{the robust objective}}
Our robust objective $L^{W}_\rho(\theta)$ is defined via a worst-case expectation over the global oracle uncertainty set $U^{W}(P^\star,\rho)$. The following assumption specifies the log-linear SAIL comparison model and enables an exact reduction of the inner worst-case problem to an explicit regularizer.

\begin{assumption}[Log-linear policy class]
\label{ass:loglinear_sail}
Let $\psi:X\times Y\to\R^d$ be a known $d$-dimensional feature map.
We consider the class of log-linear (softmax) policies
\[
\Pi
:=
\left\{
\pi_\theta:\ 
\pi_\theta(y\mid x)
=
\frac{\exp\bigl(\theta^\top \psi(x,y)\bigr)}
{\sum_{y'\in Y}\exp\bigl(\theta^\top \psi(x,y')\bigr)}
\right\}.
\]
\end{assumption}

\begin{remark}
The pointwise worst-case Bernoulli maximization does not rely on log-linearity. For any policy class with well-defined likelihood ratios, it gives
$$
L_\rho^W(\theta)=L_{\mathrm{SAIL}}(\theta)+\rho\beta\,\mathbb E_{z\sim d_\theta}|h_\theta(z)|.
$$
This assumption is used to rewrite the log-ratio score as the explicit linear pairwise score
$$
h_\theta(x,y_1,y_2)=(\theta-\theta_{\rm ref})^\top(\psi(x,y_1)-\psi(x,y_2)),
$$
which yields the penalty $R(\theta)$ analyzed below. The subsequent weak-convexity and oracle-complexity results rely on this linear-score structure together with the boundedness assumptions.
\end{remark}
For notational convenience, let $\theta_{\mathrm{ref}}\in\Theta$ denote the fixed parameter of the reference (SFT) policy, i.e,
\begin{equation}
    \label{ref_theta}
    \pi_{\theta_{\mathrm{ref}}} \;=\; \pi_{\mathrm{SFT}} .
\end{equation}

\begin{theorem}[Decomposition of $L^{W}_\rho(\theta)$]
\label{thm:decomposition_LW}
Recall $L^{\mathrm{SAIL}}(\theta)$ denotes the non-robust SAIL objective and $L^{W}_\rho(\theta)$ denotes the robust objective defined by the worst-case oracle in $U^{W}(P^\star,\rho)$. Define the pairwise score
\[
R(\theta; x,y_1,y_2):=(\theta-\theta_{\mathrm{ref}})^\top\!\bigl(\psi(x,y_1)-\psi(x,y_2)\bigr),
\]
and the robust penalty 
\begin{equation}
\label{robust_penalty}
R(\theta):=\E_{x\sim\mu}\E_{y_1,y_2\sim\pi_\theta(\cdot\mid x)}\left[\bigl|R(\theta;x,y_1,y_2)\bigr|\right].
\end{equation}

Then under Assumptions \ref{ass:nondegenerate_oracle}, \ref{ass:loglinear_sail}, the robust objective admits the exact decomposition
\begin{equation}
\label{eq:LW_decomp}
L^{W}_\rho(\theta)=L^{\mathrm{SAIL}}(\theta)+\lambda R(\theta),
\qquad
\lambda:=\rho\beta.
\end{equation}
\end{theorem}
\begin{proof}[Proof sketch of Theorem \ref{thm:decomposition_LW}]
Pointwise endpoint maximization in $p(z)$, together with the log-linear cancellation
$h_\theta(x,y_1,y_2)=(\theta-\theta_{\mathrm{ref}})^\top(\psi(x,y_1)-\psi(x,y_2))$, gives the result; see Appendix~\ref{app:A}.
\end{proof}


\begin{remark}[Interpretation of the decomposition]\label{rem:decomp_interp}
Theorem~\ref{thm:decomposition_LW} separates the robust objective into a nominal fitting term $L^{\mathrm{SAIL}}(\theta)$ and an explicit robustness penalty $\lambda R(\theta)$ that depends on the policy-induced sampling distribution. The penalty $R(\theta)$ measures the expected magnitude of the pairwise score $R(\theta;x,y_1,y_2)$ over i.i.d.\ response pairs $(y_1,y_2)\sim\pi_\theta(\cdot\mid x)$, and thus quantifies sensitivity of the likelihood to adversarial perturbations of the pointwise preference probability. The uncertainty radius enters only through the linear prefactor $\lambda=\rho\beta$: increasing $\rho$ monotonically strengthens the penalty, while $\rho\downarrow 0$ recovers the nominal objective $L^{\mathrm{SAIL}}(\theta)$.
\end{remark}

\subsection{Regularity of
\texorpdfstring{$L^{\mathrm{SAIL}}$}{the SAIL loss}}

Prior analyses of preference-based objectives often invoke stronger global conditions (e.g., PL-type geometries \citep{gaur2026on}). Here we treat $L^{\mathrm{SAIL}}(\theta)$ under the following smoothness assumption.

\begin{assumption}[Smooth SAIL objective]
\label{ass:smooth_sail}
The nominal SAIL objective $L^{\mathrm{SAIL}}$ has $L_{\mathrm{SAIL}}$-Lipschitz gradient:
\[
\|\nabla L^{\mathrm{SAIL}}(\theta)-\nabla L^{\mathrm{SAIL}}(\theta')\|\le L_{\mathrm{SAIL}}\|\theta-\theta'\|
\;\forall\,\theta,\theta'\in\R^d.
\]
\end{assumption}

This assumption is used in Theorem~\ref{thm:weak_convexity_composite} and in the convergence analysis.
\begin{remark}
    Assumption \ref{ass:smooth_sail} is a standard smoothness condition in first-order optimization, and it is routinely imposed in analyses of gradient-based methods for both convex and nonconvex objectives, including stochastic settings \citep{nesterov2013introductory,ghadimi2013stochastic,bottou2018optimization}. We adopt it here as a mild regularity requirement on $L^{\mathrm{SAIL}}$.
\end{remark}


\subsection{Weak convexity of
\texorpdfstring{$R(\theta)$}{R(theta)}}
It remains to understand the regularity of the robustness penalty $R(\theta)$. Although $|R(\theta;x,y_1,y_2)|$ is convex in $\theta$ pointwise, the expectation under the policy-dependent law $d_\theta$ destroys convexity in general; a one-dimensional counterexample is given in Appendix~\ref{app:B:nonconvex-example}. We therefore control the curvature of $R(\theta)$ through weak convexity: $f$ is $\kappa$-weakly convex if $f(\cdot)+\tfrac{\kappa}{2}\|\cdot\|^2$ is convex.
First we introduce several assumptions required for the analysis.

\begin{assumption}[Finite response set and bounded features]
\label{ass:finiteY_boundedpsi}
The response space $Y$ is finite, and the feature map satisfies $\|\psi(x,y)\|\le B_{\psi}< \infty$ for all $(x,y)\in X\times Y$.
\end{assumption}

\begin{remark}
This assumption is a proof-level regularity condition used to obtain uniform score-function and curvature bounds for the policy-dependent penalty $R(\theta)$. Such assumptions are standard in analyses of log-linear and softmax models \citep{agarwal2021theory, zhang2022convergence}.
\end{remark}




\begin{assumption}[Bounded feasible set]
\label{ass:bounded_theta}
The feasible set $\Theta$ is nonempty, closed, convex, and bounded. In particular, with $\theta_{\mathrm{ref}}$ as in Eq.\eqref{ref_theta}, define
\[
D\ :=\ \sup_{\theta\in\Theta}\|\theta-\theta_{\mathrm{ref}}\|  < \infty.
\]
\end{assumption}
This assumption makes projection well-defined \citep{10.5555/2028633} and supplies the boundedness needed in Theorem~\ref{thm:weak_convexity_R}.



\begin{theorem}[Weak convexity of the robust penalty]
\label{thm:weak_convexity_R}
Under Assumptions \ref{ass:loglinear_sail}, \ref{ass:finiteY_boundedpsi}, \ref{ass:bounded_theta}, the robust penalty $R:\Theta\to\R$ defined in Eq. \eqref{robust_penalty} is $\kappa_R$-weakly convex on $\Theta$, with
\[
\kappa_R \le 16B_\psi^2 + 4DB_\psi^3.
\]
\end{theorem}
\begin{proof}[Proof sketch of Theorem \ref{thm:weak_convexity_R}]
The proof smooths $|\cdot|$ by $\phi_\varepsilon(u)=\sqrt{u^2+\varepsilon^2}$, applies score-function calculus to the policy-dependent expectation, and bounds the Hessian uniformly using bounded features and bounded $\theta$. See Appendix~\ref{app:B:thm47}.
\end{proof}
\begin{theorem}[Weak convexity of the composite objective]
\label{thm:weak_convexity_composite}
Under Assumptions \ref{ass:nondegenerate_oracle}, \ref{ass:loglinear_sail}, \ref{ass:smooth_sail}, \ref{ass:finiteY_boundedpsi}, \ref{ass:bounded_theta}, the robust objective
$L^{W}_\rho(\theta)=L^{\mathrm{SAIL}}(\theta)+\lambda R(\theta)$ is $\kappa$-weakly convex on $\Theta$, with
\[
\kappa:=L_{\mathrm{SAIL}}+\lambda\kappa_R,
\qquad \lambda=\rho\beta.
\]
\end{theorem}

\begin{proof}[Proof sketch of Theorem \ref{thm:weak_convexity_composite}]
Sums of $L$-smooth and $\kappa_R$-weakly convex functions are $(L+\lambda\kappa_R)$-weakly convex; see Appendix~\ref{app:B:thm48}.
\end{proof}
\subsection{Algorithm development}
We optimize the oracle-robust alignment objective $L^W_{\rho}(\theta)$ over the bounded convex parameter set $\Theta$. To enforce feasibility explicitly, we consider the constrained objective
\begin{equation}
F(\theta)\ :=\ L^W_{\rho}(\theta) + I_{\Theta}(\theta),
\qquad
I_{\Theta}(\theta)\ :=\
\begin{cases}
0, & \theta\in\Theta,\\
+\infty, & \theta\notin\Theta.
\end{cases}
\label{eq:constrained_objective}
\end{equation}

Because $F$ may be nonsmooth and only weakly convex, we measure first-order stationarity using the Moreau envelope \citep{moreau1965proximite}.

\begin{definition}[Moreau envelope and proximal point]
\label{def:moreau_envelope}
Fix $\lambda_{\mathrm{env}}\in(0,1/\kappa)$. The Moreau envelope of $F$ with parameter $\lambda_{\mathrm{env}}$ is
\begin{equation}
F_{\lambda_{\mathrm{env}}}(\theta)
\ :=\
\min_{u\in\R^d}\left\{
F(u) + \frac{1}{2\lambda_{\mathrm{env}}}\|u-\theta\|^2
\right\},
\label{eq:moreau_def}
\end{equation}
and the associated proximal point mapping is
\begin{equation}
\prox_{\lambda_{\mathrm{env}}F}(\theta)
\ :=\
\arg\min_{u\in\R^d}\left\{
F(u) + \frac{1}{2\lambda_{\mathrm{env}}}\|u-\theta\|^2
\right\}.
\label{eq:prox_def}
\end{equation}
When convenient, we write $\hat{\theta}:=\prox_{\lambda_{\mathrm{env}}F}(\theta)$.
\end{definition}

\begin{lemma}[Properties of the Moreau envelope \citep{davis2018stochastic}]
\label{lem:moreau_properties}
Assume $F$ is $\kappa$-weakly convex and bounded below, and let $\lambda_{\mathrm{env}}\in(0,1/\kappa)$. Then:
\begin{enumerate}
\item $F_{\lambda_{\mathrm{env}}}$ is finite and continuously differentiable on $\R^d$.
\item Its gradient is $\nabla F_{\lambda_{\mathrm{env}}}(\theta) =
\frac{1}{\lambda_{\mathrm{env}}}\bigl(\theta-\hat{\theta}\bigr),$
and in particular,
\begin{equation}
\|\theta-\hat{\theta}\|
\ =\
\lambda_{\mathrm{env}}\|\nabla F_{\lambda_{\mathrm{env}}}(\theta)\|.
\label{eq:moreau_residual}
\end{equation}
\item The gradient $\nabla F_{\lambda_{\mathrm{env}}}$ is Lipschitz with constant
\begin{equation}
L_{\mathrm{env}}
\ :=\
\frac{1}{\lambda_{\mathrm{env}}(1-\kappa\lambda_{\mathrm{env}})}.
\label{eq:Lenv}
\end{equation}
\item (Approximate stationarity) For $\hat{\theta}=\prox_{\lambda_{\mathrm{env}}F}(\theta)$,
\begin{equation}
\dist\bigl(0,\partial F(\hat{\theta})\bigr)
\ \le\
\|\nabla F_{\lambda_{\mathrm{env}}}(\theta)\|.
\label{eq:approx_stationarity}
\end{equation}
\end{enumerate}
\end{lemma}

Lemma~\ref{lem:moreau_properties} motivates $\|\nabla F_{\lambda_{\mathrm{env}}}(\theta)\|$ as a smooth stationarity measure: controlling $\|\nabla F_{\lambda_{\mathrm{env}}}(\theta)\|$ ensures (i) $\theta$ is close to its proximal point $\hat{\theta}$ via~\eqref{eq:moreau_residual} and (ii) $\hat{\theta}$ is nearly stationary for the original constrained problem via~\eqref{eq:approx_stationarity}. The envelope parameter $\lambda_{\mathrm{env}}$ governs the smoothing--bias trade-off: smaller $\lambda_{\mathrm{env}}$ reduces smoothing bias but increases $L_{\mathrm{env}}$ in~\eqref{eq:Lenv}.


\begin{lemma}[Stochastic gradient/subgradient oracles]
\label{lem:constructed_oracles}
Suppose Assumptions~\ref{ass:loglinear_sail},
\ref{ass:finiteY_boundedpsi}, and~\ref{ass:bounded_theta} hold.
At iteration $t$, Algorithm~\ref{alg:rscgd} samples a mini-batch $Z_t=\{z_i\}_{i=1}^B$ with i.i.d.\ draws $z_i\sim d_{\theta_t}$.
There exist stochastic oracles $G^{\mathrm{SAIL}}$ and $G^{R}$ satisfying , for all $\theta,u\in\Theta$,
\[
\E\!\left[G^{\mathrm{SAIL}}(\theta;Z_t)\mid \theta\right]=\nabla L^{\mathrm{SAIL}}(\theta),
\]
\begin{equation}
R(u)\ge R(\theta)
+\langle\E[G^R(\theta;Z_t)\mid\theta],u-\theta\rangle
-\frac{\kappa_R}{2}\|u-\theta\|^2,
\label{eq:constructed_R_lower_model}
\end{equation}
Moreover, there exist constants $\sigma^2_{\mathrm{SAIL}},\sigma^2_{R}<\infty$ such that,
\[
\E\!\left[\bigl\|G^{\mathrm{SAIL}}(\theta;Z_t)-\nabla L^{\mathrm{SAIL}}(\theta)\bigr\|^2\mid\theta\right]\le \frac{\sigma^2_{\mathrm{SAIL}}}{B},\]
\[
\E\!\left[
\|G^R(\theta;Z_t)-\E[G^R(\theta;Z_t)\mid\theta]\|^2
\mid\theta
\right]
\le\frac{\sigma_R^2}{B}
\]
\end{lemma}

We minimize $F$ in~\eqref{eq:constrained_objective} by a projected stochastic composite gradient method (Algorithm~\ref{alg:rscgd}). At each iterate $\theta_t$, we draw a mini-batch under the policy-induced sampling $d_{\theta_t}$, query preference labels for the SAIL term, and combine the two stochastic oracles with weight $\lambda=\rho\beta$:
\begin{equation}
G(\theta_t;Z_t)\ :=\ G^{\mathrm{SAIL}}(\theta_t;Z_t) + \lambda\, G^{R}(\theta_t;Z_t).
\label{eq:composite_oracle}
\end{equation}
Projection onto $\Theta$ keeps the iterates feasible and the constrained objective $F$ well-defined:
\begin{equation}
\theta_{t+1}
\ :=\
\proj_{\Theta}\!\left(\theta_t - \eta\, G(\theta_t;Z_t)\right),
\label{eq:projected_update}
\end{equation}
where $\eta>0$ is a stepsize.

\begin{algorithm}[t]
\caption{Robust Stochastic Composite Gradient Descent (R-SCGD)}
\label{alg:rscgd}
\begin{algorithmic}[1]
\STATE \textbf{Input:} $\theta_0\in\Theta$, stepsize $\eta>0$, horizon $T$, batch size $B$, weight $\lambda=\rho\beta$.
\FOR{$t=0,1,\dots,T-1$}
\STATE Sample a mini-batch $Z_t=\{(x_i,y_{1,i},y_{2,i})\}_{i=1}^B$ with $x_i\sim\mu$ and $y_{1,i},y_{2,i}\stackrel{\mathrm{i.i.d.}}{\sim}\pi_{\theta_t}(\cdot\mid x_i)$.
\STATE Query preference labels from the nominal training oracle $P^\star$ for the sampled pairs to construct $(y_{w,i},y_{\ell,i})$ (equivalently, construct binary labels $y_i\in\{0,1\}$).
\STATE Compute stochastic oracles $G^{\mathrm{SAIL}}(\theta_t;Z_t)$ and $G^{R}(\theta_t;Z_t)$ of Lemma~\ref{lem:constructed_oracles}, and form $G(\theta_t;Z_t)=G^{\mathrm{SAIL}}(\theta_t;Z_t)+\lambda G^{R}(\theta_t;Z_t)$.
\STATE Update $\theta_{t+1} \leftarrow \proj_{\Theta}\!\bigl(\theta_t-\eta\,G(\theta_t;Z_t)\bigr)$.
\ENDFOR
\STATE Sample $R\sim\mathrm{Unif}\{0,1,\dots,T-1\}$.
\STATE \textbf{Output:} $\theta_R$ .
\end{algorithmic}
\end{algorithm}
\section{Analysis}

We analyze Algorithm~\ref{alg:rscgd} for the constrained objective $F$ in~\eqref{eq:constrained_objective} using the Moreau-envelope stationarity measure from Definition~\ref{def:moreau_envelope}. Throughout, we use the weak convexity constant $\kappa$ from Theorem \ref{thm:weak_convexity_composite} and assume the smoothness condition on $L^{\mathrm{SAIL}}$ stated in Assumption \ref{ass:smooth_sail}.

\begin{assumption}[Lower bounded objective]
\label{ass:lower_bounded}
The constrained objective $F(\theta)=L^W_{\rho}(\theta)+I_{\Theta}(\theta)$ is proper, lower semicontinuous, and bounded from below on $\R^d$; denote
\[
F_{\inf}\ :=\ \inf_{\theta\in\R^d} F(\theta)\ >\ -\infty.
\]
\end{assumption}

\textbf{Main Result.} We next state the convergence guarantee for Algorithm~\ref{alg:rscgd} and the resulting oracle complexity bound.

For the convergence bound, define the deterministic gradient-size
constants
\[
G_{\nabla\mathrm{SAIL}}
:=
2\beta B_\psi
+
4B_\psi\bigl(\log 2+2\beta D B_\psi\bigr),
\qquad
G_{\partial R}
:=
2B_\psi+8D B_\psi^2 .
\]
By Lemma~\ref{lem:constructed_oracles}, \(G(\theta;Z)=G^{\rm SAIL}(\theta;Z)+\lambda G^R(\theta;Z)\) satisfies
the uniform conditional second-moment bound
\[
\mathbb E\!\left[\|G(\theta;Z)\|^2\mid \theta\right]
\le
 G_{\mathrm{tot}}^2,
\qquad
G_{\mathrm{tot}}^2
:=
4\left(
G_{\nabla\mathrm{SAIL}}^2
+
\lambda^2G_{\partial R}^2
+
\frac{\sigma_{\rm SAIL}^2+\lambda^2\sigma_R^2}{B}
\right),
\]
for every \(\theta\in\Theta\), where \(B\) is the mini-batch size and
\(\lambda=\rho\beta\). The proof of this bound is given in
Lemma~\ref{lem:Gtot_from_A6_eq22}.

\begin{theorem}[Convergence rate for the Moreau envelope]
\label{thm:moreau_convergence}
Let Assumptions~\ref{ass:nondegenerate_oracle}--\ref{ass:lower_bounded} hold and let $\kappa$ be the weak convexity constant from Theorem \ref{thm:weak_convexity_composite}. Fix any $\lambda_{\mathrm{env}}\in(0,1/\kappa)$ and run Algorithm~\ref{alg:rscgd} for $T$ iterations with stepsize $\eta>0$. Let $R\sim\mathrm{Unif}\{0,1,\dots,T-1\}$ denote the output index and $\theta_R$ the corresponding iterate. Then
\begin{equation}
\E\!\left[\bigl\|\nabla F_{\lambda_{\mathrm{env}}}(\theta_R)\bigr\|^2\right]
\ \le\
\frac{F_{\lambda_{\mathrm{env}}}(\theta_0)-F_{\inf}}{\eta(1-\kappa\lambda_{\mathrm{env}})\,T}
\;+\;
\frac{L_{\mathrm{env}}\,\eta}{2(1-\kappa\lambda_{\mathrm{env}})}\,G_{\mathrm{tot}}^2,
\label{eq:moreau_rate}
\end{equation}
where $L_{\mathrm{env}}$ is defined in~\eqref{eq:Lenv}.
\end{theorem}


\begin{corollary}[Sample complexity for envelope stationarity]
\label{cor:sample_complexity}
Under the conditions of Theorem~\ref{thm:moreau_convergence}, fix $\lambda_{\mathrm{env}}\in(0,1/\kappa)$ and set
\[
\eta\ :=\ \sqrt{\frac{2\lambda_{\mathrm{env}}(1-\kappa\lambda_{\mathrm{env}})\bigl(F_{\lambda_{\mathrm{env}}}(\theta_0)-F_{\inf}\bigr)}{G_{\mathrm{tot}}^2\,T}}.
\]
Then Algorithm~\ref{alg:rscgd} guarantees
\[
\E\!\left[\bigl\|\nabla F_{\lambda_{\mathrm{env}}}(\theta_R)\bigr\|^2\right]
\ \le\
\sqrt{\frac{2\,G_{\mathrm{tot}}^2\bigl(F_{\lambda_{\mathrm{env}}}(\theta_0)-F_{\inf}\bigr)}{\lambda_{\mathrm{env}}(1-\kappa\lambda_{\mathrm{env}})^3\,T}}.
\]
Consequently, to achieve $\E[\|\nabla F_{\lambda_{\mathrm{env}}}(\theta_R)\|^2]\le \varepsilon$, it suffices to take
\begin{equation*}
   \begin{aligned}
      T\ &\ge\
\frac{2\,G_{\mathrm{tot}}^2\bigl(F_{\lambda_{\mathrm{env}}}(\theta_0)-F_{\inf}\bigr)}{\lambda_{\mathrm{env}}(1-\kappa\lambda_{\mathrm{env}})^3\,\varepsilon^2} \\
&= \frac{8\left(
C + \frac{\sigma_{\mathrm{SAIL}}^2+\lambda^2\sigma_R^2}{B}
\right)\bigl(F_{\lambda_{\mathrm{env}}}(\theta_0)-F_{\inf}\bigr)}{\lambda_{\mathrm{env}}(1-\kappa\lambda_{\mathrm{env}})^3\,\varepsilon^2}.
   \end{aligned} 
\end{equation*}
where $C = G_{\nabla\mathrm{SAIL}}^2 + \lambda^2 G_{\partial R}^2$. If $B=\widetilde{O}(1)$, the total oracle  complexity is $BT=\widetilde{O}(\varepsilon^{-2})$.

\begin{remark}[Role of smoothness]
\label{rem:smoothness}
Theorem~\ref{thm:moreau_convergence} does not require $L^W_\rho$ to be
smooth; $R$, and hence $L^W_\rho$, may be nonsmooth. Its proof uses the
$\kappa$-weak convexity of $L^W_\rho$ on $\Theta$ and the lower model of
Lemma~\ref{lem:constructed_oracles}. The only smoothness assumption is
Assumption~\ref{ass:smooth_sail}, which is imposed on the nominal term
$L^{\mathrm{SAIL}}$ rather than derived from the log-linear structure;
together with Theorem~\ref{thm:weak_convexity_R}, it gives the constant
$\kappa=L_{\mathrm{SAIL}}+\rho\beta\kappa_R$ of
Theorem~\ref{thm:weak_convexity_composite}. The smoothness of the Moreau
envelope $F_{\lambda_{\mathrm{env}}}$ then follows from
Lemma~\ref{lem:moreau_properties}.
\end{remark}

\begin{remark}[Choice of the radius]
\label{rem:radius}
Theorem~\ref{thm:moreau_convergence} and Corollary~\ref{cor:sample_complexity}
hold for every fixed radius $\rho\in(0,\delta)$, with $\lambda=\rho\beta$ and
the corresponding $\kappa$ and $G_{\mathrm{tot}}$. The guarantee concerns the
objective $L^W_\rho$ defined by the chosen radius; different radii may lead to
different stationary points. Since the actual level of misspecification is
unknown in practice, we treat $\rho$ as a model-selection parameter;
Appendices~\ref{app:radius-selection} and~\ref{app:utility-selection} give
validation-based and utility-constrained selection rules.
\end{remark}

\end{corollary}


\section{On the practical role of the proximal point}\label{sec:prox_point_practice}
Our analysis measures stationarity of the \emph{constrained} robust objective
\[
F(\theta)\ :=\ L_\rho^{W}(\theta)+I_{\Theta}(\theta),
\]
via the Moreau envelope $F_{\lambda_{\mathrm{env}}}$ with parameter $\lambda_{\mathrm{env}}\in(0,1/\kappa)$ (Definition \ref{def:moreau_envelope}), where $\kappa$ is the weak convexity constant from Theorem \ref{thm:weak_convexity_composite}. Algorithm~1 outputs a uniformly random iterate $\theta_R$. The theory (Theorem \ref{thm:moreau_convergence}) certifies \emph{envelope stationarity} at $\theta_R$ by controlling $\|\nabla F_{\lambda_{\mathrm{env}}}(\theta_R)\|_2$, rather than directly bounding $\dist(0,\partial F(\theta_R))$. The interpretation is instead through the associated proximal point
\[
\hat\theta_R\ :=\ \operatorname{prox}_{\lambda_{\mathrm{env}}F}(\theta_R)
\ \in\ \arg\min_{u\in\mathbb{R}^d}\Big\{F(u)+\tfrac{1}{2\lambda_{\mathrm{env}}}\|u-\theta_R\|_2^2\Big\},
\]
which is unique since $F$ is $\kappa$-weakly convex and $\lambda_{\mathrm{env}}<1/\kappa$.

Lemma \ref{lem:moreau_properties} provides the key link between envelope stationarity and proximity:
\[
\|\theta_R-\hat\theta_R\|_2 \ =\ \lambda_{\mathrm{env}}\|\nabla F_{\lambda_{\mathrm{env}}}(\theta_R)\|_2,
\]
and moreover $\dist(0,\partial F(\hat\theta_R))\le \|\nabla F_{\lambda_{\mathrm{env}}}(\theta_R)\|_2$.
Thus, small $\|\nabla F_{\lambda_{\mathrm{env}}}(\theta_R)\|_2$ simultaneously implies (i) $\theta_R$ lies in a small neighborhood of $\hat\theta_R$, and (ii) $\hat\theta_R$ is nearly first-order stationary for the original constrained objective.

A practical issue is that computing $\hat\theta_R$ exactly still requires solving the (generally nontrivial) strongly convex proximal subproblem above. Even when $\|\nabla F_{\lambda_{\mathrm{env}}}(\theta_R)\|_2$ is small (hence $\hat\theta_R$ is nearby), proximal operators frequently do not admit closed forms and are evaluated via iterative inner solvers, which can be expensive in realistic models \citep{article}.

Lemma \ref{lem:no-prox-output} shows that while $\dist(0,\partial F(\theta))\ge (1-\kappa\lambda_{\mathrm{env}})\|\nabla F_{\lambda_{\mathrm{env}}}(\theta)\|$ always holds, under $\kappa$-weak convexity alone there is \emph{no} universal constant $C$ such that
$\dist(0,\partial F(\theta))\le C\|\nabla F_{\lambda_{\mathrm{env}}}(\theta)\|$
for all $\theta$ (even when $\kappa=0$). This motivates phrasing guarantees in terms of the proximal point $\hat\theta_R$ rather than the raw iterate $\theta_R$.

In practice, one can compute an \emph{inexact} proximal refinement $\bar\theta_R\approx \hat\theta_R$ via a warm-started inner loop on
$\Psi_R(\theta):=F(\theta)+\tfrac{1}{2\lambda_{\mathrm{env}}}\|\theta-\theta_R\|_2^2$, where proximal subproblems are solved only up
to a prescribed accuracy (e.g., \citep{schmidt2011convergence}). Appendix~\ref{app:prox-computation} formalizes this viewpoint for our setting and provides a residual-based stopping rule.
Lemma \ref{lem:inexact-prox-stationarity} yields the practical takeaway: if $\bar\theta_R$ satisfies the proximal residual condition
$\dist(0,\partial \Psi_R(\bar\theta_R))\le \varepsilon_{\mathrm{prox}}$,
then
\[
\dist(0,\partial F(\bar\theta_R))\ \le\ \|\nabla F_{\lambda_{\mathrm{env}}}(\theta_R)\|_2\ +\ \varepsilon_{\mathrm{prox}}\ +\ \tfrac{1}{\lambda_{\mathrm{env}}}\|\bar\theta_R-\hat\theta_R\|_2.
\]
Consequently, controlling $\|\nabla F_{\lambda_{\mathrm{env}}}(\theta_R)\|_2$ together with the inexactness terms $\varepsilon_{\mathrm{prox}}$ and $\|\bar{\theta}_R-\hat{\theta}_R\|_2$ yields an explicit near-stationarity certificate for the original objective $F$.
\input{experiments_finite_response}

\section{Conclusion}\label{sec:conclusion}
We studied oracle-robust online alignment of language models, where preference feedback is collected on-policy from a nominal training oracle $P^\star$ that may be misspecified. 
We introduced a pointwise oracle uncertainty set $U^{W}(P^\star,\rho)$ and defined the robust objective $L^{W}_{\rho}(\theta)$ as the worst-case negative log-likelihood over $P\in U^{W}(P^\star,\rho)$. Our main novelty is an exact closed-form decomposition of this otherwise difficult minimax objective: $L^{W}_{\rho}(\theta)=L_{\mathrm{SAIL}}(\theta)+\lambda R(\theta)$ with $\lambda=\rho\beta$. 
For log-linear policies over a finite response set, we establish an $\widetilde{O}(\varepsilon^{-2})$ stochastic-oracle complexity bound for reaching a squared Moreau-envelope stationarity, which implies proximity to a nearly stationary point of the original robust objective; as a special case, setting $\rho=0$ recovers a convergence-to-stationarity guarantee for optimizing $L_{\mathrm{SAIL}}(\theta)$. Our experiments show that, under structured oracle misspecification, Robust SAIL trades a small increase in ideal-oracle loss for a sharp reduction in oracle sensitivity, improving worst-case loss and high-stress AUC across all tested radii and paired seeds.

%% file: experiments_finite_response.tex
\section{Experiments}
\label{sec:experiments}

We conduct a controlled finite-response experiment to validate the mechanism predicted by our robust decomposition: oracle-robust SAIL should reduce policy sensitivity to adversarial oracle perturbations, possibly at a small cost in ideal-oracle fitting.

\paragraph{Experimental setup.}
For each prompt $x$, we fix a candidate set $\mathcal{Y}_x=\{y_1,\ldots,y_K\}$ with $K=12$ and train a $d=7$-dimensional log-linear policy over this set. Comparisons are sampled on-policy from $d_\theta(x,y_i,y_j)=\mu(x)\pi_\theta(y_i\mid x)\pi_\theta(y_j\mid x)$. Let $p_{\mathrm{id}}(z)=P^{\mathrm{id}}(y_i\succ y_j\mid x)$ for $z=(x,y_i,y_j)$. The observed training oracle is a bounded misspecification of the ideal oracle, $p_{\mathrm{obs}}(z)=p_{\mathrm{id}}(z)+\rho_{\mathrm{mis}}b(z)$, where $b$ is an antisymmetric nuisance-bias direction: it is induced by a style/redundancy score orthogonal to the true-reward direction, pushes preference toward the response with larger nuisance score only when that ordering conflicts with the true reward, and is normalized per prompt so that $|b(z)|\le 1$. See Appendix~\ref{app:experiments} for details. We sweep $\rho_{\mathrm{mis}}\in\{0.05,0.075,0.10\}$ over $10$ paired seeds.

Nominal SAIL corresponds to $q=0$ and minimizes $L_{\mathrm{SAIL}}(\theta;P^{\mathrm{obs}})$. Robust SAIL minimizes $L_{\mathrm{SAIL}}(\theta;P^{\mathrm{obs}})+\rho_{\mathrm{alg}}\beta R(\theta)$, where $\beta=0.1$, $q=\rho_{\mathrm{alg}}/\rho_{\mathrm{mis}}$, $R(\theta)=\mathbb{E}_{z\sim d_\theta}|h_\theta(z)|$, and $h_\theta(z)=\log\frac{\pi_\theta(y_i\mid x)}{\pi_{\mathrm{ref}}(y_i\mid x)}-\log\frac{\pi_\theta(y_j\mid x)}{\pi_{\mathrm{ref}}(y_j\mid x)}$. The operating point $q^\star$ is selected by seed-pooled validation high-stress AUC; validation selects $q^\star=1$ for all three radii, so the main robust comparison uses $q^\star=1$. See Appendix~\ref{app:radius-selection} for details.

\paragraph{True-to-worst oracle stress metric.}
The policy-specific worst oracle chooses the loss-maximizing endpoint around $p_{\mathrm{obs}}$: $p^\theta_{\mathrm{wc}}(z)=\mathrm{clip}_{[0,1]}\{p_{\mathrm{obs}}(z)+\rho_{\mathrm{mis}}\operatorname{sgn}(\ell^1_\theta(z)-\ell^0_\theta(z))\}$. The true-to-worst path is $p^\theta_\alpha(z)=(1-\alpha)p_{\mathrm{id}}(z)+\alpha p^\theta_{\mathrm{wc}}(z)$ for $\alpha\in[0,1]$, and $L_{\mathrm{stress}}(\theta;\alpha)=\mathbb{E}_{z\sim d_\theta}[p^\theta_\alpha(z)\ell^1_\theta(z)+(1-p^\theta_\alpha(z))\ell^0_\theta(z)]$. We report $L_{\mathrm{id}}=L_{\mathrm{stress}}(\theta;0)$, $L_{\mathrm{wc}}=L_{\mathrm{stress}}(\theta;1)$, and $\mathrm{AUC}_{0.5:1}=\frac{1}{6}\sum_{\alpha\in\{0.5,0.6,0.7,0.8,0.9,1.0\}}L_{\mathrm{stress}}(\theta;\alpha)$. All loss-style metrics are lower-is-better.

Ignoring endpoint clipping when the interval is already feasible, the worst-case loss decomposes as
\begin{equation}
\label{eq:exp_lwc_decomposition}
L_{\mathrm{wc}}(\theta)
=
L_{\mathrm{id}}(\theta)
+
\rho_{\mathrm{mis}}\,
\mathbb{E}_{z\sim d_\theta}\!\left[
b(z)\bigl(\ell^1_\theta(z)-\ell^0_\theta(z)\bigr)
\right]
+
\rho_{\mathrm{mis}}\beta R(\theta).
\end{equation}
Thus $L_{\mathrm{wc}}$ separates ideal-oracle performance, the fixed observed-oracle bias, and the algorithmic robustness penalty controlled by $\rho_{\mathrm{alg}}R(\theta)$.

\paragraph{Results. }

Table~\ref{tab:main} and Fig.~\ref{fig:true_to_worst_stress}  show the main comparison at $\rho_{\mathrm{mis}}=0.075$. Robust SAIL slightly sacrifices ideal-oracle fitting, increasing $L_{\mathrm{id}}$ from $0.6923$ to $0.6931$, but reduces the sensitivity penalty from $R(\theta)=0.4928$ to $0.0025$. As a result, its stress curve is nearly flat: the worst-case endpoint improves from $L_{\mathrm{wc}}=0.6960$ to $0.6932$, and the AUC improves from $0.6951$ to $0.6932$. The paired curves cross at $\alpha\approx0.23$, meaning Robust SAIL becomes preferable once the evaluation oracle moves about $23\%$ of the way from the ideal oracle toward the worst-case oracle. Across $\rho_{\mathrm{mis}}\in\{0.05,0.075,0.10\}$, validation selects $q^\star=1$ and Robust SAIL improves $L_{\mathrm{wc}}$, $\mathrm{AUC}_{0.5:1}$, and $R(\theta)$ on all $10/10$ paired seeds; this is a controlled radius-calibration result, with deployment-time radius selection discussed in Appendix~\ref{app:experiments}.



\begin{table}[t]
\centering
\vspace{-0.5em}
\caption{Main stress metrics at $\rho_{\mathrm{mis}}=0.075$, averaged over 10 paired seeds. Lower is better for all columns.}
\label{tab:main}
\small
\setlength{\tabcolsep}{8pt}
\begin{tabular}{lcccc}
\toprule
Method & $L_{\mathrm{id}}$ & $R(\theta)$ & $L_{\mathrm{wc}}$ & $\mathrm{AUC}_{0.5:1}$ \\
\midrule
Nominal SAIL ($q=0$) & 0.6923 & 0.4928 & 0.6960 & 0.6951 \\
Robust SAIL (ours) ($q^\star=1$) & 0.6931 & 0.0025 & 0.6932 & 0.6932 \\
\bottomrule
\end{tabular}
\vspace{-0.8em}
\end{table}

\begin{figure}[t]
\centering
\vspace{-0.5em}
\includegraphics[width=.8\linewidth]{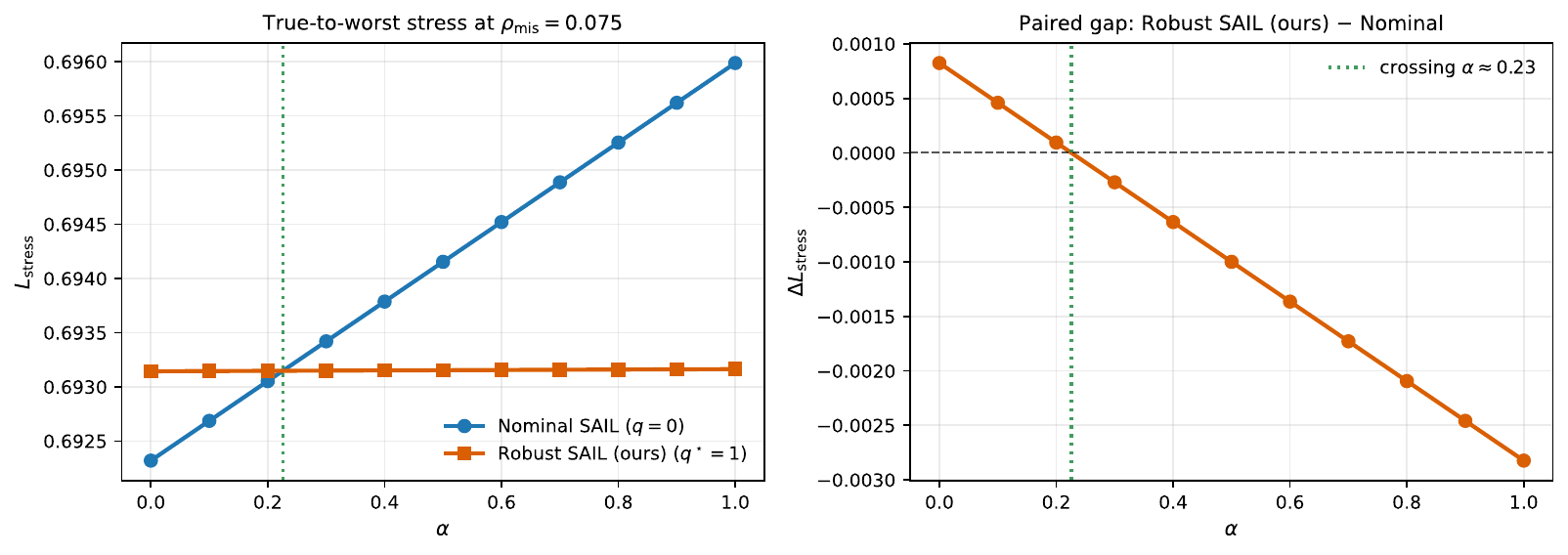}
\vspace{-0.1in}
\caption{Left: true-to-worst oracle stress curve at $\rho_{\mathrm{mis}}=0.075$. Right: Robust SAIL sacrifices ideal-oracle fitting but remains nearly flat along the stress path and becomes better after $\alpha\approx0.23$.}
\label{fig:true_to_worst_stress}
\vspace{-0.8em}
\end{figure}

%% file: appendixA.tex
\newpage
\section{Related Work}
\label{sec:related}
Bilevel optimization provides a principled abstraction for hierarchical learning problems and is increasingly relevant to online alignment, where preference learning and policy optimization are coupled through data generated by the current policy \citep{christiano2017deep,ouyang2022training,chakraborty2023parl}. SAIL reduces this online RLHF bilevel structure to a single-level first-order surrogate through reward-policy equivalence \citep{ding2024sail}, while recent theory studies nonconvex bilevel RL and penalty-based bilevel methods \citep{gaur2026on,shen2023penalty,shen2025principled}. Recent online
preference-alignment methods further study exploration and active preference
elicitation for on-policy feedback collection, including count-based online
preference optimization and self-exploring language models
\citep{bai2025online,zhang2024self}.  These works model policy-dependent data collection, but they do not optimize against an explicit oracle uncertainty set.

Robust alignment work studies preference and data shift through distributionally robust objectives, including robust DPO variants with Wasserstein or KL uncertainty sets and group-robust preference optimization \citep{rafailov2023dpo,xu2025robustdpo,wu2025drdpo,ramesh2024grpo}. These methods build on the broader DRO literature \citep{rahimian2019distributionally,mohajerin2018data,kuhn2019wasserstein,duchi2021learning}, but are typically formulated for fixed offline datasets. Our setting differs by placing oracle uncertainty inside the on-policy sampling loop.

\section{Proof of Theorem~\ref{thm:decomposition_LW}:
Decomposition of
\texorpdfstring{$L^W_{\rho}(\theta)$}{the robust objective}}
\label{app:A}

\subsection{A pointwise maximization identity}\label{app:A:pointwise}

\begin{lemma}[Pointwise worst-case Bernoulli perturbation]\label{lem:A:pointwise}
Fix any $p^\star\in[0,1]$, any $\rho\ge 0$ such that $[p^\star-\rho,p^\star+\rho]\subseteq[0,1]$, and any real numbers
$a,b\in\mathbb{R}$. Then
\begin{equation}\label{eq:A:pointwise}
\sup_{p\in[p^\star-\rho,p^\star+\rho]}\Bigl(p\,a+(1-p)\,b\Bigr)
= p^\star a+(1-p^\star)b+\rho\,|a-b|.
\end{equation}
Moreover, an optimizer is
$p^\star+\rho$ if $a\ge b$ and $p^\star-\rho$ if $a<b$.
\end{lemma}

\begin{proof}
The map $p\mapsto p\,a+(1-p)\,b=b+p(a-b)$ is affine in $p$, hence it attains its maximum over the interval
$[p^\star-\rho,p^\star+\rho]$ at an endpoint.
If $a-b\ge 0$ the maximizer is $p^\star+\rho$, giving value
$b+(p^\star+\rho)(a-b)=p^\star a+(1-p^\star)b+\rho(a-b)$.
If $a-b<0$ the maximizer is $p^\star-\rho$, giving
$b+(p^\star-\rho)(a-b)=p^\star a+(1-p^\star)b+\rho(b-a)$.
Combining the two cases yields~\eqref{eq:A:pointwise}.
\end{proof}

\subsection{Proof of Theorem~\ref{thm:decomposition_LW}}\label{app:A:thm41}

\begin{proof}[Proof of Theorem~\ref{thm:decomposition_LW}]
Recall:
(i) the policy-dependent sampling law $d_\theta$ on $\mathcal{Z}=\mathcal{X}\times\mathcal{Y}\times\mathcal{Y}$
(Eq. \eqref{eq:data_gen}),
(ii) the pointwise oracle uncertainty set $U^W(P^\star,\rho)$ (Eq. \eqref{eq:global_U}),
(iii) the robust objective $L^W_\rho(\theta)$ (Eq. \eqref{eq:Lrobust_def}),
and (iv) the SAIL objective $L^{\mathrm{SAIL}}(\theta)$.

\paragraph{Step 1: Rewrite the robust objective pointwise in $p(z)$.}
For $z=(x,y_1,y_2)$, let $p(z):=P(1\mid z)=P(y_1\succ y_2\mid x)$ and $p^\star(z):=P^\star(1\mid z)$.
By Eq. \eqref{eq:global_U}, $P\in U^W(P^\star,\rho)$ iff for all $z$,
\begin{equation}\label{eq:A:p_interval}
|p(z)-p^\star(z)|\le \rho.
\end{equation}
Assumption~1 ensures $p^\star(z)\in[\delta,1-\delta]$ and $\rho\in(0,\delta)$, so
$[p^\star(z)-\rho,p^\star(z)+\rho]\subseteq[0,1]$ for all $z$.

Define the pairwise logit as
\[
h_\theta(z)\equiv h_\theta(x,y_1,y_2)
:= \log\frac{\pi_\theta(y_1\mid x)}{\pi_{\mathrm{SFT}}(y_1\mid x)}
-\log\frac{\pi_\theta(y_2\mid x)}{\pi_{\mathrm{SFT}}(y_2\mid x)}.
\]
For notational convenience, set the two label-conditional losses
\begin{equation}\label{eq:A:ell10}
\ell_\theta^1(z):=-\log\sigma\bigl(\beta\,h_\theta(z)\bigr),\qquad
\ell_\theta^0(z):=-\log\sigma\bigl(\beta\,h_\theta(x,y_2,y_1)\bigr)
=-\log\sigma\bigl(-\beta\,h_\theta(z)\bigr),
\end{equation}
where we used $h_\theta(x,y_2,y_1)=-h_\theta(x,y_1,y_2)$.

Then, for any oracle $P$ with Bernoulli parameter $p(z)$, the conditional risk equals
\[
\mathbb{E}_{y\sim P(\cdot\mid z)}\bigl[\ell_\theta(z,y)\bigr]
= p(z)\,\ell_\theta^1(z)+(1-p(z))\,\ell_\theta^0(z)
=: \ell\bigl(p(z);h_\theta(z)\bigr).
\]
Therefore $L^W_\rho(\theta)$ can be written as
\begin{equation}\label{eq:A:robust_pointwise}
L^W_\rho(\theta)
=\sup_{P\in U^W(P^\star,\rho)}\ \mathbb{E}_{z\sim d_\theta}
\Bigl[p(z)\,\ell_\theta^1(z)+(1-p(z))\,\ell_\theta^0(z)\Bigr].
\end{equation}

Because the constraint~\eqref{eq:A:p_interval} is pointwise in $z$ and the objective is an integral (expectation)
of a pointwise affine function of $p(z)$, the worst-case oracle can be chosen pointwise in $z$.
Equivalently, the supremum in~\eqref{eq:A:robust_pointwise} equals the expectation of the pointwise supremum:
\begin{equation}\label{eq:A:swap_sup}
L^W_\rho(\theta)
= \mathbb{E}_{z\sim d_\theta}\left[\ \sup_{p\in[p^\star(z)-\rho,\,p^\star(z)+\rho]}
\Bigl(p\,\ell_\theta^1(z)+(1-p)\,\ell_\theta^0(z)\Bigr)\right].
\end{equation}
By Lemma~\ref{lem:A:pointwise}, an optimizer is always an endpoint, hence a measurable selector can be obtained
by taking $p(z)=p^\star(z)+\rho$ when $\ell_\theta^1(z)\ge \ell_\theta^0(z)$ and $p(z)=p^\star(z)-\rho$ otherwise.

Apply Lemma~\ref{lem:A:pointwise} to~\eqref{eq:A:swap_sup} with
$a=\ell_\theta^1(z)$, $b=\ell_\theta^0(z)$ and $p^\star=p^\star(z)$. This gives
\begin{align}
L^W_\rho(\theta)
&= \mathbb{E}_{z\sim d_\theta}\Bigl[p^\star(z)\,\ell_\theta^1(z)+(1-p^\star(z))\,\ell_\theta^0(z)\Bigr]
+ \rho\,\mathbb{E}_{z\sim d_\theta}\Bigl[|\ell_\theta^1(z)-\ell_\theta^0(z)|\Bigr].
\label{eq:A:two_terms}
\end{align}
The first expectation in~\eqref{eq:A:two_terms} is exactly $L^{\mathrm{SAIL}}(\theta)$.

For the second term, use the identity
\[
\log\sigma(t)-\log\sigma(-t)=t
\quad\text{(since $\sigma(t)/\sigma(-t)=e^t$).}
\]
With $t=\beta h_\theta(z)$ and definitions~\eqref{eq:A:ell10}, we obtain
\[
\ell_\theta^1(z)-\ell_\theta^0(z)
= -\log\sigma(\beta h_\theta(z))+\log\sigma(-\beta h_\theta(z))
= -\beta h_\theta(z).
\]
Hence $|\ell_\theta^1(z)-\ell_\theta^0(z)|=\beta |h_\theta(z)|$, and~\eqref{eq:A:two_terms} becomes
\begin{equation}\label{eq:A:decomp_h}
L^W_\rho(\theta)
= L_{\mathrm{SAIL}}(\theta)+\rho\beta\,\mathbb{E}_{z\sim d_\theta}\bigl[|h_\theta(z)|\bigr].
\end{equation}
Define $\lambda:=\rho\beta$.

Under Assumption \ref{ass:loglinear_sail}, for fixed $x$ we have the log-linear / softmax form
\[
\pi_\theta(y\mid x)=\frac{\exp(\theta^\top\psi(x,y))}{\sum_{y'\in\mathcal{Y}}\exp(\theta^\top\psi(x,y'))}.
\]
By Eq. \eqref{ref_theta},\ $\pi_{\mathrm{SFT}}=\pi_{\theta_{\mathrm{ref}}}$ for some $\theta_{\mathrm{ref}}$.
Then
\[
\log\frac{\pi_\theta(y\mid x)}{\pi_{\mathrm{SFT}}(y\mid x)}
=(\theta-\theta_{\mathrm{ref}})^\top\psi(x,y)
-\log\sum_{y'}e^{\theta^\top\psi(x,y')}
+\log\sum_{y'}e^{\theta_{\mathrm{ref}}^\top\psi(x,y')}.
\]
Taking the difference between $y_1$ and $y_2$ cancels the log-partition terms, yielding
\begin{equation}\label{eq:A:h_equals_s}
h_\theta(x,y_1,y_2)=(\theta-\theta_{\mathrm{ref}})^\top\bigl(\psi(x,y_1)-\psi(x,y_2)\bigr).
\end{equation}

Then~\eqref{eq:A:h_equals_s} shows $h_\theta(z)=R(\theta; x,y_1,y_2)$ where $R(\theta; x,y_1,y_2)$ is defined in Theorem \ref{thm:decomposition_LW}, and thus the robust correction term in~\eqref{eq:A:decomp_h}
is precisely $\lambda R(\theta)$ with
\[
R(\theta):=\mathbb{E}_{z\sim d_\theta}\bigl[|R(\theta; x,y_1,y_2)|\bigr].
\] 
Substituting into~\eqref{eq:A:decomp_h} gives the claimed decomposition
\[
L^W_\rho(\theta)=L_{\mathrm{SAIL}}(\theta)+\lambda R(\theta),
\qquad \lambda=\rho\beta.
\]
This completes the proof.
\end{proof}

%% file: appendixB.tex
\section{Weak convexity: proofs of Theorems~\ref{thm:weak_convexity_R} and~\ref{thm:weak_convexity_composite}}\label{app:B}

\subsection{A one-dimensional nonconvexity example}\label{app:B:nonconvex-example}

Consider a single prompt $X=\{x\}$, two responses $Y=\{a,b\}$, and $d=1$. Fix $\theta_{\mathrm{ref}}=0$ and define features by $\psi(x,a)=1$ and $\psi(x,b)=0$, so that
\[
R(\theta;x,y_1,y_2)=\theta\bigl(\psi(x,y_1)-\psi(x,y_2)\bigr)\in\{0,\pm\theta\}.
\]
Let $\pi_\theta(a\mid x)=\sigma(\theta)$ and $\pi_\theta(b\mid x)=\sigma(-\theta)$, and sample $y_1,y_2\iid\pi_\theta(\cdot\mid x)$. Then
\[
R(\theta)=|\theta|\Pr(y_1\neq y_2)=2|\theta|\sigma(\theta)\sigma(-\theta)
=\frac{2|\theta|e^\theta}{(1+e^\theta)^2}.
\]
Convexity would imply $R(1)\le \tfrac12(R(0)+R(2))=\tfrac12R(2)$, but direct calculation gives
\[
R(1)-\tfrac12R(2)
=\frac{2e(e-1)(e^3-1)}{(1+e)^2(1+e^2)^2}>0.
\]
Thus $R$ is not convex in general.

\subsection{Auxiliary lemmas}\label{app:B:aux}

\begin{lemma}[Policy-regularity constants for the log-linear policy class]\label{lem:B:policy_reg}
Assume Assumptions~\ref{ass:loglinear_sail} and~\ref{ass:finiteY_boundedpsi}.
Define the policy score $g_\theta(x,y):=\nabla_\theta\log\pi_\theta(y\mid x)$.
Then for all $\theta,\theta'\in\mathbb{R}^d$ and all $(x,y)\in\mathcal{X}\times\mathcal{Y}$,
\[
\|g_\theta(x,y)\|_2\le G,\qquad
\|g_\theta(x,y)-g_{\theta'}(x,y)\|_2\le M\|\theta-\theta'\|_2,
\]
with explicit constants $G:=2B_\psi$ and $M:=B_\psi^2$.

Moreover, for $z=(x,y_1,y_2)$ and $d_\theta(z)=\mu(x)\pi_\theta(y_1\mid x)\pi_\theta(y_2\mid x)$, define score \citep{williams1992simple}
\[
S_\theta(z):=\nabla_\theta\log d_\theta(z)=g_\theta(x,y_1)+g_\theta(x,y_2).
\]
Then
\[
\|S_\theta(z)\|_2\le 2G=4B_\psi,\qquad
\|\nabla_\theta S_\theta(z)\|_{\mathrm{op}}\le 2M=2B_\psi^2.
\]
\end{lemma}

\begin{proof}
Under Assumption~\ref{ass:loglinear_sail}, for fixed $x$ the policy is
\[
\pi_\theta(y\mid x)=\frac{\exp(\theta^\top\psi(x,y))}{\sum_{y'\in\mathcal{Y}}\exp(\theta^\top\psi(x,y'))}.
\]
Hence
\[
\log\pi_\theta(y\mid x)=\theta^\top\psi(x,y)-\log\sum_{y'\in\mathcal{Y}}e^{\theta^\top\psi(x,y')},
\]
so differentiating gives
\[
g_\theta(x,y)
=\psi(x,y)-\sum_{y'\in\mathcal{Y}}\pi_\theta(y'\mid x)\psi(x,y')
= \psi(x,y)-\mu_\theta(x),
\quad
\mu_\theta(x):=\mathbb{E}_{y'\sim\pi_\theta(\cdot\mid x)}[\psi(x,y')].
\]
By Assumption~\ref{ass:finiteY_boundedpsi}, $\|\psi(x,y)\|_2\le B_\psi$ and Jensen implies $\|\mu_\theta(x)\|_2\le B_\psi$,
hence $\|g_\theta(x,y)\|_2\le 2B_\psi$. This gives $G:=2B_\psi$.

Next, $\nabla_\theta g_\theta(x,y)=-\nabla_\theta\mu_\theta(x)$.
A standard calculation yields that the Jacobian of $\mu_\theta(x)$ is the covariance matrix
\[
\nabla_\theta\mu_\theta(x)=\mathrm{Cov}_{y'\sim\pi_\theta(\cdot\mid x)}\!\bigl(\psi(x,y')\bigr).
\]
For any unit vector $v\in\mathbb{R}^d$,
\[
v^\top\nabla_\theta\mu_\theta(x)v
=\mathrm{Var}\bigl(v^\top\psi(x,y')\bigr)
\le \mathbb{E}\bigl[(v^\top\psi(x,y'))^2\bigr]
\le \mathbb{E}\|\psi(x,y')\|_2^2
\le B_\psi^2.
\]
Taking the supremum over $\|v\|_2=1$ yields $\|\nabla_\theta\mu_\theta(x)\|_{\mathrm{op}}\le B_\psi^2$,
so $\|\nabla_\theta g_\theta(x,y)\|_{\mathrm{op}}\le B_\psi^2$ as well. By the mean value theorem,
\[
\|g_\theta(x,y)-g_{\theta'}(x,y)\|_2\le B_\psi^2\|\theta-\theta'\|_2,
\]
so we may take $M:=B_\psi^2$.

Finally, since $S_\theta(z)=g_\theta(x,y_1)+g_\theta(x,y_2)$, the bounds
$\|S_\theta(z)\|_2\le 2G$ and $\|\nabla_\theta S_\theta(z)\|_{\mathrm{op}}\le 2M$ follow by triangle inequality.
\end{proof}

\begin{lemma}[Bounded Fisher information]\label{lem:B:fisher}
Define the Fisher information matrix
\[
\mathcal{F}(\theta)
:=\mathbb{E}_{x\sim\mu,\ y\sim \pi_\theta(\cdot\mid x)}\bigl[g_\theta(x,y)g_\theta(x,y)^\top\bigr].
\]
Then $\|\mathcal{F}(\theta)\|_{\mathrm{op}}\le G^2$ for all $\theta\in\Theta$.
\end{lemma}

\begin{proof}
For any unit vector $u\in\mathbb{R}^d$,
\[
u^\top\mathcal{F}(\theta)u
=\mathbb{E}\bigl[\langle u,g_\theta(x,y)\rangle^2\bigr]
\le \mathbb{E}\bigl[\|g_\theta(x,y)\|_2^2\bigr]
\le G^2.
\]
Taking the supremum over $\|u\|_2=1$ yields $\|\mathcal{F}(\theta)\|_{\mathrm{op}}\le G^2$.
\end{proof}

\begin{remark}\label{rem:B:opnorm}
In the weak convexity analysis of $R(\theta)$,
we differentiate an expectation taken under the \emph{$\theta$-dependent} law $d_\theta$.
The resulting Hessian formula contains curvature terms involving the score and its Jacobian,
notably $r_{\varepsilon,\theta}(z)\,\nabla_\theta S_\theta(z)$ where
$S_\theta(z)=\nabla_\theta\log d_\theta(z)$ and $r_{\varepsilon,\theta}(z)$ is a smoothed absolute value.
To obtain a \emph{uniform} lower bound $\nabla^2 R_\varepsilon(\theta)\succeq -\kappa_\varepsilon I$ (hence weak convexity),
we must control the worst-case quadratic form $v^\top(\nabla_\theta S_\theta(z))v$ uniformly over $\|v\|_2=1$.
This is exactly an operator-norm control. Lemma~\ref{lem:B:fisher} exemplifies this type of bound
for second-moment (information) matrices, and the proof of Lemma~\ref{lem:B:policy_reg}
uses the same operator-norm reasoning to bound $\|\nabla_\theta S_\theta(z)\|_{\mathrm{op}}$,
which is the quantity that directly enters the proof of Theorem \ref{thm:weak_convexity_R}.
\end{remark}

\subsection{Proof of Theorem~\ref{thm:weak_convexity_R} (weak convexity of the robust penalty)}\label{app:B:thm47}

\begin{proof}[Proof of Theorem~\ref{thm:weak_convexity_R}]
Recall the robust penalty 
\[
R(\theta):=\mathbb{E}_{z\sim d_\theta}\bigl[|R(\theta; x,y_1,y_2)|\bigr],
\;
R(\theta; x,y_1,y_2):=(\theta-\theta_{\mathrm{ref}})^\top\bigl(\psi(x,y_1)-\psi(x,y_2)\bigr),
\]
  where $z=(x,y_1,y_2)$.
By~\eqref{eq:A:h_equals_s}, $R(\theta;x,y_1,y_2)=h_\theta(z)$, the pairwise log-ratio score.
\paragraph{Step 1: Smooth approximation.}
Fix $\varepsilon>0$ and define $\phi_\varepsilon(u):=\sqrt{u^2+\varepsilon^2}$.
Set
\[
R_\varepsilon(\theta):=\mathbb{E}_{z\sim d_\theta}\bigl[\phi_\varepsilon(h_\theta(z))\bigr].
\]
Since $\mathcal{Y}$ is finite, the expectation over $d_\theta$ is a finite sum, so $R_\varepsilon$
is twice continuously differentiable and we may differentiate under the sum.

We will show that $\nabla^2 R_\varepsilon(\theta)\succeq -\kappa_\varepsilon I$ for all $\theta\in\Theta$ with
\begin{equation}\label{eq:B:kappa_eps}
\kappa_\varepsilon:=8GB_\psi + 4MDB_\psi + 2M\varepsilon,
\end{equation}
where $G,M$ are as in Lemma~\ref{lem:B:policy_reg}. Since $R_\varepsilon$ is $C^2$ and has a global Hessian lower bound,
the function $R_\varepsilon(\cdot)+\frac{\kappa_\varepsilon}{2}\|\cdot\|_2^2$ has positive semidefinite Hessian and is convex;
equivalently, $R_\varepsilon$ is $\kappa_\varepsilon$-weakly convex. 

\paragraph{Step 2: Uniform bounds.}
Define $\Delta\psi(z):=\psi(x,y_1)-\psi(x,y_2)$. By Assumption \ref{ass:finiteY_boundedpsi},
\[
\|\Delta\psi(z)\|_2\le \|\psi(x,y_1)\|_2+\|\psi(x,y_2)\|_2\le 2B_\psi.
\]
By Assumption \ref{ass:bounded_theta}, $\|\theta-\theta_{\mathrm{ref}}\|_2\le D$ for all $\theta\in\Theta$, hence
\[
|h_\theta(z)|\le \|\theta-\theta_{\mathrm{ref}}\|_2\,\|\Delta\psi(z)\|_2 \le 2DB_\psi.
\]
Moreover, $\phi_\varepsilon(u)\le |u|+\varepsilon$, so
\begin{equation}\label{eq:B:r_bound}
0\le \phi_\varepsilon(h_\theta(z))\le 2DB_\psi+\varepsilon.
\end{equation}

\paragraph{Step 3: Score-function calculus for $d_\theta$.}
Write $d_\theta(z)=\mu(x)\pi_\theta(y_1\mid x)\pi_\theta(y_2\mid x)$ and define the score \citep{williams1992simple}
\[
S_\theta(z):=\nabla_\theta\log d_\theta(z).
\]
Then $S_\theta(z)=g_\theta(x,y_1)+g_\theta(x,y_2)$, hence Lemma~\ref{lem:B:policy_reg} gives
\begin{equation}\label{eq:B:score_bounds}
\|S_\theta(z)\|_2\le 2G,\qquad \|\nabla_\theta S_\theta(z)\|_{\mathrm{op}}\le 2M.
\end{equation}

For the integrand $r_{\varepsilon,\theta}(z):=\phi_\varepsilon(h_\theta(z))$, note that
\[
\nabla_\theta h_\theta(z)=\Delta\psi(z),
\qquad
\nabla_\theta r_{\varepsilon,\theta}(z)
=\phi_\varepsilon'(h_\theta(z))\,\Delta\psi(z),
\]
and since $|\phi_\varepsilon'(u)|=|u|/\sqrt{u^2+\varepsilon^2}\le 1$,
\begin{equation}\label{eq:B:grad_r_bound}
\|\nabla_\theta r_{\varepsilon,\theta}(z)\|_2 \le \|\Delta\psi(z)\|_2 \le 2B_\psi.
\end{equation}
Also,
\[
\nabla_\theta^2 r_{\varepsilon,\theta}(z)
=\phi_\varepsilon''(h_\theta(z))\,\Delta\psi(z)\Delta\psi(z)^\top,
\qquad
\phi_\varepsilon''(u)=\frac{\varepsilon^2}{(u^2+\varepsilon^2)^{3/2}}\ge 0,
\]
so $\nabla_\theta^2 r_{\varepsilon,\theta}(z)\succeq 0$ for all $(\theta,z)$.

Because $\mathcal{Y}$ is finite, we can differentiate $R_\varepsilon(\theta)=\sum_z d_\theta(z)\,r_{\varepsilon,\theta}(z)$
term-by-term and use $\nabla_\theta d_\theta(z)=d_\theta(z)S_\theta(z)$ to obtain:
\begin{align}
\nabla_\theta R_\varepsilon(\theta)
&=\mathbb{E}_{z\sim d_\theta}\bigl[\nabla_\theta r_{\varepsilon,\theta}(z)+r_{\varepsilon,\theta}(z)S_\theta(z)\bigr],\label{eq:B:grad_formula}\\
\nabla_\theta^2 R_\varepsilon(\theta)
&=\mathbb{E}_{z\sim d_\theta}\Bigl[
\nabla_\theta^2 r_{\varepsilon,\theta}(z)
+S_\theta(z)\nabla_\theta r_{\varepsilon,\theta}(z)^\top
+\nabla_\theta r_{\varepsilon,\theta}(z)S_\theta(z)^\top
\notag\\
  &\qquad\qquad+r_{\varepsilon,\theta}(z)\bigl(S_\theta(z)S_\theta(z)^\top+\nabla_\theta S_\theta(z)\bigr)
\Bigr].\label{eq:B:hess_formula}
\end{align}

\paragraph{Step 4: Hessian lower bound.}
Fix any unit vector $u\in\mathbb{R}^d$.
Using~\eqref{eq:B:hess_formula} and the facts that
$\nabla_\theta^2 r_{\varepsilon,\theta}(z)\succeq 0$ and $S_\theta(z)S_\theta(z)^\top\succeq 0$, we have
\begin{align*}
u^\top\nabla_\theta^2 R_\varepsilon(\theta)u
&\ge \mathbb{E}_{z\sim d_\theta}\Bigl[
u^\top\bigl(S_\theta\nabla_\theta r_{\varepsilon,\theta}^\top+\nabla_\theta r_{\varepsilon,\theta}S_\theta^\top\bigr)u
+r_{\varepsilon,\theta}(z)\,u^\top(\nabla_\theta S_\theta(z))u
\Bigr].
\end{align*}
For the first term,
\[
u^\top\bigl(S_\theta\nabla_\theta r_{\varepsilon,\theta}^\top+\nabla_\theta r_{\varepsilon,\theta}S_\theta^\top\bigr)u
=2\langle u,S_\theta(z)\rangle\,\langle u,\nabla_\theta r_{\varepsilon,\theta}(z)\rangle
\ge -2\|S_\theta(z)\|_2\,\|\nabla_\theta r_{\varepsilon,\theta}(z)\|_2,
\]
and using~\eqref{eq:B:score_bounds} and~\eqref{eq:B:grad_r_bound} gives the pointwise bound
\[
u^\top\bigl(S_\theta\nabla_\theta r_{\varepsilon,\theta}^\top+\nabla_\theta r_{\varepsilon,\theta}S_\theta^\top\bigr)u
\ge -2(2G)(2B_\psi)=-8GB_\psi.
\]
For the second term, by~\eqref{eq:B:score_bounds},
$u^\top(\nabla_\theta S_\theta(z))u\ge -\|\nabla_\theta S_\theta(z)\|_{\mathrm{op}}\ge -2M$,
so with~\eqref{eq:B:r_bound} we have
\[
r_{\varepsilon,\theta}(z)\,u^\top(\nabla_\theta S_\theta(z))u
\ge -(2DB_\psi+\varepsilon)\,2M
= -4MDB_\psi -2M\varepsilon.
\]
Combining the two bounds and taking expectation yields
\[
u^\top\nabla_\theta^2 R_\varepsilon(\theta)u \ge -(8GB_\psi+4MDB_\psi+2M\varepsilon)= -\kappa_\varepsilon,
\]
for all unit $u$, hence $\nabla_\theta^2 R_\varepsilon(\theta)\succeq -\kappa_\varepsilon I$.

\paragraph{Step 5: Pass to the nonsmooth limit $\varepsilon\downarrow 0$.}
For any $u\in\mathbb{R}$, $0\le \phi_\varepsilon(u)-|u|\le \varepsilon$, hence
$0\le R_\varepsilon(\theta)-R(\theta)\le \varepsilon$ for all $\theta$.
Thus $R_\varepsilon\to R$ uniformly as $\varepsilon\downarrow 0$.
Since each $R_\varepsilon$ is $\kappa_\varepsilon$-weakly convex and $\kappa_\varepsilon\downarrow \kappa_R:=8GB_\psi+4MDB_\psi$,
taking $\varepsilon\downarrow 0$ in the defining weak convexity inequality yields that $R$ is $\kappa_R$-weakly convex.

Finally, by Lemma~\ref{lem:B:policy_reg}, we may take $G=2B_\psi$ and $M=B_\psi^2$, giving
\[
\kappa_R \le 8(2B_\psi)B_\psi + 4(B_\psi^2)D B_\psi
=16B_\psi^2 + 4DB_\psi^3,
\]
which is the claimed bound in Theorem~\ref{thm:weak_convexity_R}.
\end{proof}

\subsection{Proof of Theorem~\ref{thm:weak_convexity_composite}}\label{app:B:thm48}

\begin{proof}[Proof of Theorem~\ref{thm:weak_convexity_composite}]
By Theorem \ref{thm:decomposition_LW} we have the exact decomposition 
\[
L^W_\rho(\theta)=L_{\mathrm{SAIL}}(\theta)+\lambda R(\theta),
\qquad \lambda=\rho\beta.
\]
Under Assumption \ref{ass:smooth_sail}, $L_{\mathrm{SAIL}}$ has $L_{\mathrm{SAIL}}$-Lipschitz gradient on $\mathbb{R}^d$.
A basic fact is that any $L$-smooth function is $L$-weakly convex. This can be verified directly from the second-order characterization: if $f$ is $L$-smooth, then $\nabla^2 f(\theta)\succeq -L I$. Equivalently, the function
\[
\theta \;\mapsto\; f(\theta) + \frac{L}{2}\|\theta\|_2^2
\]
is convex.

By Theorem \ref{thm:weak_convexity_R}, $R$ is $\kappa_R$-weakly convex on $\Theta$. Scaling preserves weak convexity, so $\lambda R$
is $(\lambda\kappa_R)$-weakly convex. Sums of weakly convex functions add their curvature parameters,
hence $L^W_\rho=L_{\mathrm{SAIL}}+\lambda R$ is $\kappa$-weakly convex on $\Theta$ with
\[
\kappa:=L_{\mathrm{SAIL}}+\lambda\kappa_R.
\]
This concludes the proof of Theorem \ref{thm:weak_convexity_composite}.
\end{proof}

\subsection{Proof of Lemma~\ref{lem:constructed_oracles}}
\label{app:oracle_construction}

\begin{proof}
Fix an iteration $t$ and condition on $\theta_t=\theta\in\Theta$. The complete
triples $z_i=(x_i,y_{1,i},y_{2,i})$, $i=1,\dots,B$, are then i.i.d.\ draws
from $d_\theta$. Given these triples, the preference labels are drawn
independently from the nominal oracle $P^\star$; we write
$(y_{w,i},y_{\ell,i})$ for the resulting preference ordering of
$(y_{1,i},y_{2,i})$ and regard the labels as part of the batch $Z_t$.

For each $i$, define
\[
\begin{gathered}
\Delta\psi_i:=\psi(x_i,y_{1,i})-\psi(x_i,y_{2,i}),
\qquad
h_i:=(\theta-\theta_{\mathrm{ref}})^\top\Delta\psi_i,\\
S_i:=\nabla_\theta\log\pi_\theta(y_{1,i}\mid x_i)
+\nabla_\theta\log\pi_\theta(y_{2,i}\mid x_i),
\end{gathered}
\]
and let $\Delta\psi_i^{w\ell}$ and $h_i^{w\ell}$ be defined in the same way,
with $(y_{1,i},y_{2,i})$ replaced by $(y_{w,i},y_{\ell,i})$. Let $s_0(0)=0$
and $s_0(a)=\operatorname{sgn}(a)$ for $a\neq0$, and define the sample
directions
\[
\begin{aligned}
H_i^{\mathrm{SAIL}}
&:=-\beta\sigma\bigl(-\beta h_i^{w\ell}\bigr)\Delta\psi_i^{w\ell}
-\log\sigma\bigl(\beta h_i^{w\ell}\bigr)S_i,\\
H_i^{R}
&:=s_0(h_i)\Delta\psi_i+|h_i|S_i.
\end{aligned}
\]
The oracles are their mini-batch averages,
\begin{align}
G^{\mathrm{SAIL}}(\theta;Z_t)
&:=\frac1B\sum_{i=1}^B H_i^{\mathrm{SAIL}},
\label{eq:constructed_sail_oracle}\\
G^{R}(\theta;Z_t)
&:=\frac1B\sum_{i=1}^B H_i^{R}.
\label{eq:constructed_R_oracle}
\end{align}
In both sample directions the first term is pathwise, i.e., a (sub)gradient
of the per-sample loss with the sample held fixed, while the $S_i$-term is the
score correction required because the sampling law $d_\theta$ depends on
$\theta$.

In the notation of Appendix~\ref{app:B:thm47}, $\Delta\psi_i=\Delta\psi(z_i)$
and $S_i=S_\theta(z_i)$, where $S_\theta(z)$ is the score of $d_\theta$
(Lemma~\ref{lem:B:policy_reg}). Moreover, $h_i=h_\theta(z_i)$, where
by~\eqref{eq:A:h_equals_s}
$h_\theta(z)=(\theta-\theta_{\mathrm{ref}})^\top\Delta\psi(z)$ is the
pairwise log-ratio score; thus
$R(\theta)=\E_{z\sim d_\theta}\bigl[|h_\theta(z)|\bigr]$
by~\eqref{robust_penalty}. By
Assumptions~\ref{ass:finiteY_boundedpsi} and~\ref{ass:bounded_theta}, the
Cauchy--Schwarz inequality and Lemma~\ref{lem:B:policy_reg}, for all $z\in Z$
and $\theta\in\Theta$,
\begin{equation}\label{eq:oracle_bounds}
\|\Delta\psi(z)\|\le2B_\psi,
\quad
|h_\theta(z)|\le2DB_\psi,
\quad
\|S_\theta(z)\|\le4B_\psi,
\quad
\|\nabla_\theta S_\theta(z)\|_{\mathrm{op}}\le2B_\psi^2.
\end{equation}
Since both $z_i$ and $(x_i,y_{w,i},y_{\ell,i})$ lie in $Z$, these bounds
apply to $\Delta\psi_i$, $\Delta\psi_i^{w\ell}$, $h_i$, $h_i^{w\ell}$ and
$S_i$.

We first show that $G^{\mathrm{SAIL}}$ is unbiased. For a labeled comparison
$(x,y_w,y_\ell)$, let $\Delta\psi^{w\ell}$ and $h^{w\ell}$ be defined as above
and write $\ell_\theta:=-\log\sigma(\beta h^{w\ell})$. By~\eqref{eq:Lsail_single_level}
and~\eqref{eq:A:h_equals_s}, $L^{\mathrm{SAIL}}(\theta)$ is the expectation of
$\ell_\theta$ when $z=(x,y_1,y_2)\sim d_\theta$ and $(y_w,y_\ell)$ is the
ordering drawn from $P^\star$. Since
$\nabla_\theta h^{w\ell}=\Delta\psi^{w\ell}$ and
$(\log\sigma)'(a)=\sigma(-a)$, the pathwise derivative is
\[
\nabla_\theta\ell_\theta=-\beta\sigma\bigl(-\beta h^{w\ell}\bigr)\Delta\psi^{w\ell}.
\]
The prompt distribution and the conditional nominal oracle do not depend on
$\theta$, so the score of the sampling law of the labeled comparison comes
entirely from the two responses and equals $S_\theta(z)$. Differentiating the
policy-dependent expectation therefore gives
\[
\nabla L^{\mathrm{SAIL}}(\theta)
=\E\bigl[\nabla_\theta\ell_\theta+\ell_\theta S_\theta(z)\bigm|\theta\bigr]
=\E\bigl[H_i^{\mathrm{SAIL}}\bigm|\theta\bigr].
\]
Differentiation under the expectation is justified by dominated convergence:
by~\eqref{eq:oracle_bounds} and $0\le-\log\sigma(a)\le\log2+|a|$,
\[
0\le\ell_\theta\le\log2+2\beta DB_\psi,
\qquad
\|\nabla_\theta\ell_\theta\|\le2\beta B_\psi,
\qquad
\|S_\theta(z)\|\le4B_\psi,
\]
and the same bounds hold, with $D$ replaced by $D+1$, for all $\theta$ within
distance $1$ of $\Theta$. Averaging over $i$
in~\eqref{eq:constructed_sail_oracle} yields
\[
\E\bigl[G^{\mathrm{SAIL}}(\theta;Z_t)\mid\theta\bigr]
=\nabla L^{\mathrm{SAIL}}(\theta).
\]

Next we verify~\eqref{eq:constructed_R_lower_model}. Let
$\bar g_R(\theta):=\E[H_i^R\mid\theta]$; by~\eqref{eq:constructed_R_oracle},
$\E[G^R(\theta;Z_t)\mid\theta]=\bar g_R(\theta)$. For $\varepsilon>0$, define
\[
q_\varepsilon(a):=\sqrt{a^2+\varepsilon^2}-\varepsilon,
\qquad
R_\varepsilon(\theta):=\E_{z\sim d_\theta}\bigl[q_\varepsilon(h_\theta(z))\bigr].
\]
Then
\[
0\le q_\varepsilon(a)\le|a|,
\qquad
q_\varepsilon'(a)=\frac{a}{\sqrt{a^2+\varepsilon^2}},
\qquad
q_\varepsilon''(a)=\frac{\varepsilon^2}{(a^2+\varepsilon^2)^{3/2}}\ge0;
\]
in particular $|q_\varepsilon'|\le1$, and as $\varepsilon\downarrow0$,
$q_\varepsilon(a)\to|a|$ and $q_\varepsilon'(a)\to s_0(a)$ for every
$a\in\R$, including $a=0$.

We begin with a lower model for $R_\varepsilon$ whose constant $\kappa_R$
does not depend on $\varepsilon$; it follows from the calculation of
Appendix~\ref{app:B:thm47} with $q_\varepsilon$ in place of $\phi_\varepsilon$.
Write $r_{\varepsilon,\theta}(z):=q_\varepsilon(h_\theta(z))$. Since
$h_\theta(z)$ is affine in $\theta$ with gradient $\Delta\psi(z)$,
\[
\nabla_\theta r_{\varepsilon,\theta}(z)
=q_\varepsilon'(h_\theta(z))\Delta\psi(z),
\qquad
\nabla_\theta^2 r_{\varepsilon,\theta}(z)
=q_\varepsilon''(h_\theta(z))\Delta\psi(z)\Delta\psi(z)^\top\succeq0.
\]
The identities~\eqref{eq:B:grad_formula} and~\eqref{eq:B:hess_formula} were
derived using only $\nabla_\theta d_\theta=d_\theta S_\theta$; hence they hold
for this $r_{\varepsilon,\theta}$, and, as there, $R_\varepsilon$ is twice
continuously differentiable. In~\eqref{eq:B:hess_formula}, the terms
$\nabla_\theta^2r_{\varepsilon,\theta}$ and
$r_{\varepsilon,\theta}S_\theta S_\theta^\top$ are positive semidefinite.
Hence, for every $\theta\in\Theta$ and every unit vector $v\in\R^d$,
using~\eqref{eq:oracle_bounds},
$\|\nabla_\theta r_{\varepsilon,\theta}(z)\|\le\|\Delta\psi(z)\|$ and
$0\le r_{\varepsilon,\theta}(z)\le|h_\theta(z)|$,
\[
\begin{aligned}
v^\top\nabla^2R_\varepsilon(\theta)v
&\ge-\E_{z\sim d_\theta}\Bigl[
2\|S_\theta(z)\|\,\|\nabla_\theta r_{\varepsilon,\theta}(z)\|
+r_{\varepsilon,\theta}(z)\,\|\nabla_\theta S_\theta(z)\|_{\mathrm{op}}
\Bigr]\\
&\ge-2(4B_\psi)(2B_\psi)-(2DB_\psi)(2B_\psi^2)
=-\kappa_R,
\end{aligned}
\]
where $\kappa_R=16B_\psi^2+4DB_\psi^3$ is the constant of
Theorem~\ref{thm:weak_convexity_R}. (The bound $q_\varepsilon(a)\le|a|$, in
place of $\phi_\varepsilon(a)\le|a|+\varepsilon$, removes the term
$2M\varepsilon$ in~\eqref{eq:B:kappa_eps}.) Now let $u\in\Theta$. Since
$\Theta$ is convex, the segment from $\theta$ to $u$ lies in $\Theta$, and
integrating the Hessian bound along it (Taylor's formula with integral
remainder) gives the smooth lower inequality
\begin{equation}\label{eq:oracle_smooth_lower}
R_\varepsilon(u)\ge R_\varepsilon(\theta)
+\langle\nabla R_\varepsilon(\theta),u-\theta\rangle
-\frac{\kappa_R}{2}\|u-\theta\|^2.
\end{equation}

It remains to let $\varepsilon\downarrow0$. Since
$0\le|a|-q_\varepsilon(a)\le\varepsilon$, we have $R_\varepsilon\to R$
uniformly on $\Theta$. By~\eqref{eq:B:grad_formula},
\[
\nabla R_\varepsilon(\theta)
=\E_{z\sim d_\theta}\Bigl[
q_\varepsilon'(h_\theta(z))\Delta\psi(z)
+q_\varepsilon(h_\theta(z))S_\theta(z)
\Bigr].
\]
As $\varepsilon\downarrow0$, the integrand converges pointwise to
$s_0(h_\theta(z))\Delta\psi(z)+|h_\theta(z)|S_\theta(z)$, whose expectation
under $d_\theta$ is $\bar g_R(\theta)$; by~\eqref{eq:oracle_bounds}, its norm
is at most $2B_\psi+(2DB_\psi)(4B_\psi)$, uniformly in $\varepsilon$.
Dominated convergence therefore gives
$\nabla R_\varepsilon(\theta)\to\bar g_R(\theta)$. Letting
$\varepsilon\downarrow0$ in~\eqref{eq:oracle_smooth_lower}, we obtain
\[
R(u)\ge R(\theta)+\langle\bar g_R(\theta),u-\theta\rangle
-\frac{\kappa_R}{2}\|u-\theta\|^2,
\]
which is~\eqref{eq:constructed_R_lower_model}, because
$\bar g_R(\theta)=\E[G^R(\theta;Z_t)\mid\theta]$. The argument uses only the
segment from $\theta$ to $u$, so~\eqref{eq:constructed_R_lower_model} holds
for all $\theta,u\in\Theta$, including boundary points of $\Theta$.

We now bound the conditional moments. Recall the constants
$G_{\partial R}$ and $G_{\nabla\mathrm{SAIL}}$ defined before
Theorem~\ref{thm:moreau_convergence}. By~\eqref{eq:oracle_bounds},
$0<\sigma<1$ and $0\le-\log\sigma(a)\le\log2+|a|$, we have
\[
\begin{aligned}
\|H_i^R\|
&\le|s_0(h_i)|\,\|\Delta\psi_i\|+|h_i|\,\|S_i\|
\le2B_\psi+8DB_\psi^2=G_{\partial R},\\
\|H_i^{\mathrm{SAIL}}\|
&\le\beta\sigma\bigl(-\beta h_i^{w\ell}\bigr)\|\Delta\psi_i^{w\ell}\|
+\bigl[-\log\sigma\bigl(\beta h_i^{w\ell}\bigr)\bigr]\|S_i\|\\
&\le2\beta B_\psi+4B_\psi(\log2+2\beta DB_\psi)=G_{\nabla\mathrm{SAIL}}.
\end{aligned}
\]
Let $H_i$ denote either $H_i^{\mathrm{SAIL}}$ or $H_i^R$, and let
$m_H(\theta):=\E[H_i\mid\theta]$. Conditionally on $\theta$, the vectors
$H_1,\dots,H_B$ are i.i.d., because the complete triples are i.i.d.\ and each
label depends only on its own triple. Hence the cross terms vanish, and
\[
\begin{aligned}
\E\Bigl[\Bigl\|\frac1B\sum_{i=1}^B H_i-m_H(\theta)\Bigr\|^2\Bigm|\theta\Bigr]
&=\frac1{B^2}\sum_{i=1}^B\E\bigl[\|H_i-m_H(\theta)\|^2\mid\theta\bigr]\\
&=\frac1B\Bigl(\E\bigl[\|H_1\|^2\mid\theta\bigr]-\|m_H(\theta)\|^2\Bigr)
\le\frac1B\,\E\bigl[\|H_1\|^2\mid\theta\bigr].
\end{aligned}
\]
Since $m_H(\theta)=\nabla L^{\mathrm{SAIL}}(\theta)$ for
$H_i=H_i^{\mathrm{SAIL}}$ and $m_H(\theta)=\bar g_R(\theta)$ for
$H_i=H_i^R$, the variance bounds of the lemma hold with
\[
\sigma_{\mathrm{SAIL}}^2=G_{\nabla\mathrm{SAIL}}^2,
\qquad
\sigma_R^2=G_{\partial R}^2.
\]
Moreover, $\|H_i^R\|\le G_{\partial R}$ implies
$\|\bar g_R(\theta)\|\le G_{\partial R}$ and
$\|G^R(\theta;Z_t)\|\le G_{\partial R}$ almost surely; in particular,
$\E[\|G^R(\theta;Z_t)\|^2\mid\theta]\le G_{\partial R}^2$.

Finally, the triples and labels of iteration $t$ are sampled afresh at
$\theta_t$, from a conditional law that depends on the past only through
$\theta_t$. Hence the mean identities, the lower model and the moment bounds
above also hold conditionally on the pre-sampling history $\mathcal F_t$.
\end{proof}

\subsection{Auxiliary lemma for stochastic gradient oracle}
\begin{lemma}[Second-moment bound for gradient oracle]\label{lem:Gtot_from_A6_eq22}
Fix any mini-batch size $B\in\mathbb{N}$, and recall the composite direction in Eq.~\eqref{eq:composite_oracle},
\[
G(\theta;Z)\ :=\ G_{\mathrm{SAIL}}(\theta;Z)\ +\ \lambda\, G_{R}(\theta;Z),
\]
where $Z$ denotes a generic mini-batch of size $B$ drawn i.i.d.\ from $d_\theta$ (Lemma~\ref{lem:constructed_oracles}).
Under Assumptions \ref{ass:loglinear_sail}, \ref{ass:finiteY_boundedpsi}, \ref{ass:bounded_theta}, the composite direction has bounded conditional second moment:
for all $\theta\in\Theta$,
\[
\mathbb{E}\bigl[\|G(\theta;Z)\|_2^2 \,\big|\, \theta \bigr]\ \le\ G_{\mathrm{tot}}^2,
\]
with the explicit choice
\[
G_{\mathrm{tot}}^2
\ :=\ 4\left(
G_{\nabla\mathrm{SAIL}}^2 \;+\; \lambda^2\,G_{\partial R}^2 \;+\; \frac{\sigma_{\mathrm{SAIL}}^2+\lambda^2\sigma_R^2}{B}
\right),
\]
where
\[
G_{\nabla\mathrm{SAIL}}
\ :=\ 2\beta B_\psi \;+\; 4B_\psi\bigl(\log 2 + 2\beta D B_\psi\bigr),
\qquad
G_{\partial R}
\ :=\ 2B_\psi \;+\; 8DB_\psi^2,
\]
and $\sigma_{\mathrm{SAIL}}^2,\sigma_R^2$ are as in Lemma~\ref{lem:constructed_oracles}.
\end{lemma}

\begin{proof}

We use the stochastic oracles constructed in the proof of
Lemma~\ref{lem:constructed_oracles}, writing
$G_{\mathrm{SAIL}}=G^{\mathrm{SAIL}}$ and $G_R=G^R$,
with $Z$ denoting a generic batch of triples and preference labels.
Define
\[
\bar g_R(\theta)
:=\mathbb E[G_R(\theta;Z)\mid\theta].
\]
We bound the mean components
$\|\nabla L_{\mathrm{SAIL}}(\theta)\|_2$ and
$\|\bar g_R(\theta)\|_2$, then combine these bounds with the
centered-variance bounds in Lemma~\ref{lem:constructed_oracles}
and the definition~\eqref{eq:composite_oracle}.

\paragraph{Step 1: Uniform bounds on the pairwise logit and score.}
For $z=(x,y_1,y_2)\in\mathcal{Z}$, define $\Delta\psi(z):=\psi(x,y_1)-\psi(x,y_2)$.
By Assumption \ref{ass:finiteY_boundedpsi}, $\|\Delta\psi(z)\|_2\le 2B_\psi$. Under Assumption \ref{ass:loglinear_sail} (log-linear policy class),
the pairwise logit admits the cancellation
\[
h_\theta(z)\ =\ h_\theta(x,y_1,y_2)\ =\ (\theta-\theta_{\mathrm{ref}})^\top \Delta\psi(z)
\]
Hence, by Assumption \ref{ass:bounded_theta},
\[
|h_\theta(z)|\ \le\ \|\theta-\theta_{\mathrm{ref}}\|_2\;\|\Delta\psi(z)\|_2\ \le\ D\cdot 2B_\psi\ =\ 2DB_\psi.
\]
Moreover, recalling $d_\theta(z)=\mu(x)\pi_\theta(y_1|x)\pi_\theta(y_2|x)$ (Eq.\eqref{eq:data_gen}) and
$S_\theta(z):=\nabla_\theta\log d_\theta(z)$, Lemma \ref{lem:B:policy_reg} gives $\|S_\theta(z)\|_2\le 4B_\psi$.

\paragraph{Step 2: Uniform bound on $\|\nabla L_{\mathrm{SAIL}}(\theta)\|_2$.}
Recall that $L_{\mathrm{SAIL}}(\theta)=\mathbb{E}_{z\sim d_\theta}\bigl[p^\star\ell^1(h)+(1-p^\star)\ell^0(h)\bigr]$,
where $\ell^1(h)=-\log\sigma(\beta h)$ and
$\ell^0(h)=-\log\sigma(-\beta h)$ , and $p^\star(z)\in[0,1]$.
Define the pointwise SAIL loss
\[
\ell_\theta(z)\ :=\ p^\star(z)\,\ell^1\!\bigl(h_\theta(z)\bigr)\ +\ (1-p^\star(z))\,\ell^0\!\bigl(h_\theta(z)\bigr).
\]
Since $\ell^1(h)=\log(1+e^{-\beta h})$ and $\ell^0(h)=\log(1+e^{\beta h})$, we have for all $h\in\mathbb{R}$,
$\max\{\ell^1(h),\ell^0(h)\}=\log(1+e^{|\beta h|})\le \log 2 + |\beta h|$. Using $|h_\theta(z)|\le 2DB_\psi$,
\[
0\ \le\ \ell_\theta(z)\ \le\ \log 2 + 2\beta DB_\psi.
\]
Next, by direct differentiation, the derivatives of $\ell^1,\ell^0$ with respect to $h$ satisfy
$\bigl|\tfrac{d}{dh}\ell^1(h)\bigr|\le \beta$ and $\bigl|\tfrac{d}{dh}\ell^0(h)\bigr|\le \beta$ for all $h$.
Together with $\nabla_\theta h_\theta(z)=\Delta\psi(z)$ and $\|\Delta\psi(z)\|_2\le 2B_\psi$,
\[
\|\nabla_\theta \ell_\theta(z)\|_2\ \le\ \beta\,\|\Delta\psi(z)\|_2\ \le\ 2\beta B_\psi.
\]
Because $\mathcal{Y}$ is finite, we may apply the same score-function calculus as in Eq. \eqref{eq:B:grad_formula} to write
\[
\nabla L_{\mathrm{SAIL}}(\theta)\ =\ \mathbb{E}_{z\sim d_\theta}\!\left[\nabla_\theta \ell_\theta(z) + \ell_\theta(z)\,S_\theta(z)\right].
\]
Therefore,
\begin{align*}
\|\nabla L_{\mathrm{SAIL}}(\theta)\|_2
&\le  \mathbb{E}_{z\sim d_\theta}\!\left[\|\nabla_\theta \ell_\theta(z)\|_2 + \ell_\theta(z)\,\|S_\theta(z)\|_2\right]
\\ &\le  2\beta B_\psi + \bigl(\log 2 + 2\beta DB_\psi\bigr)\,4B_\psi
\ =:\ G_{\nabla\mathrm{SAIL}}.
\end{align*}

 \paragraph{Step 3: Bound on $\|\bar g_R(\theta)\|_2$.}
 We now bound $\|\bar g_R(\theta)\|_2$. By the proof of Lemma~\ref{lem:constructed_oracles},
\[
\bar g_R(\theta)=\mathbb{E}_{z\sim d_\theta}\bigl[s_0(h_\theta(z))\Delta\psi(z)+|h_\theta(z)|\,S_\theta(z)\bigr],
\]
with $s_0$ as defined there. Since $|s_0|\le1$, the uniform bounds above give
\[
\|\bar g_R(\theta)\|_2
\ \le\ \mathbb{E}_{z\sim d_\theta}\bigl[\|\Delta\psi(z)\|_2+|h_\theta(z)|\,\|S_\theta(z)\|_2\bigr]
\ \le\ 2B_\psi+(2DB_\psi)(4B_\psi)\ =\ G_{\partial R}.
\]

Fix $\theta\in\Theta$ and abbreviate $G_{\mathrm{SAIL}}:=G_{\mathrm{SAIL}}(\theta;Z)$ and $G_R:=G_R(\theta;Z)$.
By~\eqref{eq:composite_oracle} and $\|a+b\|_2^2\le 2\|a\|_2^2+2\|b\|_2^2$,
\[
\mathbb{E}\bigl[\|G(\theta;Z)\|_2^2\mid \theta\bigr]
\ \le\ 2\,\mathbb{E}\bigl[\|G_{\mathrm{SAIL}}\|_2^2\mid \theta\bigr] \;+\; 2\lambda^2\,\mathbb{E}\bigl[\|G_R\|_2^2\mid \theta\bigr].
\]

For the SAIL term, write $G_{\mathrm{SAIL}}=\nabla L_{\mathrm{SAIL}}(\theta) + (G_{\mathrm{SAIL}}-\nabla L_{\mathrm{SAIL}}(\theta))$ and apply
$\|a+b\|_2^2\le 2\|a\|_2^2+2\|b\|_2^2$ together with Lemma \ref{lem:constructed_oracles}:
\[
\mathbb{E}\bigl[\|G_{\mathrm{SAIL}}\|_2^2\mid \theta\bigr]
\ \le\ 2\|\nabla L_{\mathrm{SAIL}}(\theta)\|_2^2 \;+\; 2\,\mathbb{E}\bigl[\|G_{\mathrm{SAIL}}-\nabla L_{\mathrm{SAIL}}(\theta)\|_2^2\mid \theta\bigr]
\ \le\ 2G_{\nabla\mathrm{SAIL}}^2 \;+\; 2\frac{\sigma_{\mathrm{SAIL}}^2}{B}.
\]

For the robust-penalty term, write $G_R=\bar g_R(\theta)+(G_R-\bar g_R(\theta))$ and apply
$\|a+b\|_2^2\le 2\|a\|_2^2+2\|b\|_2^2$ together with the bound on $\|\bar g_R(\theta)\|_2$ and the
centered-variance bound in Lemma~\ref{lem:constructed_oracles}:
\[
\mathbb{E}\bigl[\|G_R\|_2^2\mid \theta\bigr]
\ \le\ 2\|\bar g_R(\theta)\|_2^2 \;+\; 2\,\mathbb{E}\bigl[\|G_R-\bar g_R(\theta)\|_2^2\mid \theta\bigr]
\ \le\ 2G_{\partial R}^2 \;+\; 2\frac{\sigma_R^2}{B}.
\]

Combining the last three displays yields, for all $\theta\in\Theta$,
\begin{align*}
\mathbb{E}\bigl[\|G(\theta;Z)\|_2^2\mid \theta\bigr]
&\le  2\left(2G_{\nabla\mathrm{SAIL}}^2 + 2\frac{\sigma_{\mathrm{SAIL}}^2}{B}\right)
\;+\;
2\lambda^2\left(2G_{\partial R}^2 + 2\frac{\sigma_R^2}{B}\right)
\\ &= 4\left(
G_{\nabla\mathrm{SAIL}}^2 + \lambda^2 G_{\partial R}^2 + \frac{\sigma_{\mathrm{SAIL}}^2+\lambda^2\sigma_R^2}{B}
\right).
\end{align*}
This is exactly the claimed bound with $G_{\mathrm{tot}}^2$ as stated.
\end{proof}

%% file: appendixC.tex
\section{Moreau envelope and convergence analysis}\label{app:C}

\subsection{Proof of Lemma \ref{lem:moreau_properties} (Properties of the Moreau envelope)}\label{app:C:lem411}

\begin{proof}[Proof of Lemma \ref{lem:moreau_properties}]
Fix $\lambda_{\mathrm{env}}\in(0,1/\kappa)$ and recall the definition of  the Moreau envelope and proximal point (Definition \ref{def:moreau_envelope})
\[
F_{\lambda_{\mathrm{env}}}(\theta)
:=\min_{u\in\mathbb{R}^d}\left\{F(u)+\frac{1}{2\lambda_{\mathrm{env}}}\|u-\theta\|_2^2\right\},
\]
\[
\hat\theta=\mathrm{prox}_{\lambda_{\mathrm{env}}F}(\theta)
:=\arg\min_{u\in\mathbb{R}^d}\left\{F(u)+\frac{1}{2\lambda_{\mathrm{env}}}\|u-\theta\|_2^2\right\}.
\]

\paragraph{1--2. Finiteness, uniqueness, and the gradient formula.}
Since $F$ is $\kappa$-weakly convex, the function $u\mapsto F(u)+\frac{\kappa}{2}\|u\|_2^2$ is convex.
Therefore, for any fixed $\theta$, the function
\[
u\mapsto F(u)+\frac{1}{2\lambda_{\mathrm{env}}}\|u-\theta\|_2^2
\]
is $(\frac{1}{\lambda_{\mathrm{env}}}-\kappa)$-strongly convex in $u$ (because $\lambda_{\mathrm{env}}<1/\kappa$),
hence it has a unique minimizer $\hat\theta$ and the minimum value $F_{\lambda_{\mathrm{env}}}(\theta)$ is finite.

Moreover, by first-order optimality of $\hat\theta$ for the strongly convex problem,
\begin{equation}\label{eq:C:prox_opt}
0\in \partial F(\hat\theta)+\frac{1}{\lambda_{\mathrm{env}}}(\hat\theta-\theta),
\qquad\text{i.e.,}\qquad
\frac{1}{\lambda_{\mathrm{env}}}(\theta-\hat\theta)\in \partial F(\hat\theta).
\end{equation}
Standard properties of the Moreau envelope for weakly convex functions imply that $F_{\lambda_{\mathrm{env}}}$ is
continuously differentiable and that
\[
\nabla F_{\lambda_{\mathrm{env}}}(\theta)=\frac{1}{\lambda_{\mathrm{env}}}(\theta-\hat\theta).
\]
Eq. \eqref{eq:moreau_residual} follows immediately:
$\|\theta-\hat\theta\|_2=\lambda_{\mathrm{env}}\|\nabla F_{\lambda_{\mathrm{env}}}(\theta)\|_2$.

\paragraph{3. Lipschitz continuity of $\nabla F_{\lambda_{\mathrm{env}}}$.}
A standard result for $\kappa$-weakly convex $F$ is that $F_{\lambda_{\mathrm{env}}}$ has Lipschitz gradient with constant
\[
L_{\mathrm{env}}=\frac{1}{\lambda_{\mathrm{env}}(1-\kappa\lambda_{\mathrm{env}})},
\]
which is Eq. \eqref{eq:Lenv}.

\paragraph{4. Approximate stationarity of the proximal point.}
By~\eqref{eq:C:prox_opt}, the vector $\frac{1}{\lambda_{\mathrm{env}}}(\theta-\hat\theta)$ belongs to $\partial F(\hat\theta)$.
Therefore
\[
\mathrm{dist}\bigl(0,\partial F(\hat\theta)\bigr)
\le \left\|\frac{1}{\lambda_{\mathrm{env}}}(\theta-\hat\theta)\right\|_2
=\|\nabla F_{\lambda_{\mathrm{env}}}(\theta)\|_2
\]
 This completes the proof.
\end{proof}

\subsection{Auxiliary lemmas for Theorem \ref{thm:moreau_convergence} }\label{app:C:aux}

\begin{lemma}[Monotonicity inequality]\label{lem:C:monotonicity}
Assume $F$ is $\kappa$-weakly convex and let $\lambda_{\mathrm{env}}\in(0,1/\kappa)$.
Fix any $\theta\in\mathbb{R}^d$ and let
\[
\hat\theta=\mathrm{prox}_{\lambda_{\mathrm{env}}F}(\theta),
\qquad
\xi:=\nabla F_{\lambda_{\mathrm{env}}}(\theta)=\frac{1}{\lambda_{\mathrm{env}}}(\theta-\hat\theta).
\]
Then for every $v\in\partial F(\theta)$,
\[
\langle \xi, v\rangle \ge (1-\kappa\lambda_{\mathrm{env}})\|\xi\|_2^2.
\]
\end{lemma}

\begin{proof}
By $\kappa$-weak convexity, for all $x,y$ and all $g\in\partial F(x)$,
\begin{equation}\label{eq:C:weak_subgrad}
F(y)\ge F(x)+\langle g, y-x\rangle -\frac{\kappa}{2}\|y-x\|_2^2.
\end{equation}
Apply~\eqref{eq:C:weak_subgrad} twice:

(i) with $x=\theta$, $y=\hat\theta$, and $g=v\in\partial F(\theta)$:
\begin{equation}\label{eq:C:ineq1}
F(\hat\theta)\ge F(\theta)+\langle v,\hat\theta-\theta\rangle-\frac{\kappa}{2}\|\hat\theta-\theta\|_2^2.
\end{equation}

(ii) with $x=\hat\theta$, $y=\theta$, and any $\hat v\in\partial F(\hat\theta)$:
\begin{equation}\label{eq:C:ineq2}
F(\theta)\ge F(\hat\theta)+\langle \hat v,\theta-\hat\theta\rangle-\frac{\kappa}{2}\|\theta-\hat\theta\|_2^2.
\end{equation}

By optimality of the proximal point (as in~\eqref{eq:C:prox_opt}),
we may choose $\hat v=\frac{1}{\lambda_{\mathrm{env}}}(\theta-\hat\theta)=\xi\in\partial F(\hat\theta)$.
Substituting this choice into~\eqref{eq:C:ineq2} gives
\[
F(\theta)\ge F(\hat\theta)+\langle \xi,\theta-\hat\theta\rangle-\frac{\kappa}{2}\|\theta-\hat\theta\|_2^2.
\]
Add this inequality to~\eqref{eq:C:ineq1} to eliminate $F(\theta)$ and $F(\hat\theta)$, yielding
\[
0\ge \langle v,\hat\theta-\theta\rangle+\langle \xi,\theta-\hat\theta\rangle-\kappa\|\theta-\hat\theta\|_2^2.
\]
Rearrange and use $\theta-\hat\theta=\lambda_{\mathrm{env}}\xi$:
\[
\langle \xi, \theta-\hat\theta\rangle \le \langle v, \theta-\hat\theta\rangle + \kappa\|\theta-\hat\theta\|_2^2
\quad\Longrightarrow\quad
\lambda_{\mathrm{env}}\|\xi\|_2^2 \le \lambda_{\mathrm{env}}\langle v,\xi\rangle + \kappa\lambda_{\mathrm{env}}^2\|\xi\|_2^2.
\]
Divide by $\lambda_{\mathrm{env}}>0$ and rearrange:
\[
\langle v,\xi\rangle \ge (1-\kappa\lambda_{\mathrm{env}})\|\xi\|_2^2,
\]
as claimed.
\end{proof}

\begin{lemma}[One-step inequality]\label{lem:C:one_step}
Let Assumptions~\ref{ass:nondegenerate_oracle}--\ref{ass:lower_bounded} hold and let $\kappa$ be the weak convexity constant from Theorem~\ref{thm:weak_convexity_composite}.

Fix $\lambda_{\mathrm{env}}\in(0,1/\kappa)$ and define $L_{\mathrm{env}}:=\frac{1}{\lambda_{\mathrm{env}}(1-\kappa\lambda_{\mathrm{env}})}$
as in Lemma \ref{lem:moreau_properties}.

Consider the projected update (Algorithm \ref{alg:rscgd})
\[
\theta_{t+1}:=\mathrm{proj}_{\Theta}\bigl(\theta_t-\eta\,G(\theta_t;Z_t)\bigr),
\]
where $G(\theta_t;Z_t)=G^{\mathrm{SAIL}}(\theta_t;Z_t)+\lambda G^{R}(\theta_t;Z_t)$ is formed from Lemma~\ref{lem:constructed_oracles}.

Then for any stepsize $\eta>0$,
\[
\mathbb{E}\bigl[F_{\lambda_{\mathrm{env}}}(\theta_{t+1})\mid \theta_t\bigr]
\le F_{\lambda_{\mathrm{env}}}(\theta_t)
-\eta(1-\kappa\lambda_{\mathrm{env}})\|\nabla F_{\lambda_{\mathrm{env}}}(\theta_t)\|_2^2
+\frac{L_{\mathrm{env}}\eta^2}{2}\,\mathbb{E}\bigl[\|G(\theta_t;Z_t)\|_2^2\mid \theta_t\bigr].
\]
\end{lemma}

\begin{proof}
Write $F=L^W_{\rho}+I_\Theta$, where $I_\Theta$ is the indicator of $\Theta$ (Eq.\eqref{eq:constrained_objective}).
Then the envelope can be written as
\[
F_{\lambda_{\mathrm{env}}}(z)=\min_{u\in\Theta}\left\{L^W_{\rho}(u)+\frac{1}{2\lambda_{\mathrm{env}}}\|u-z\|_2^2\right\}.
\]
Let $\tilde z:=\mathrm{proj}_\Theta(z)$. For any $u\in\Theta$, Euclidean projection optimality implies
$\langle z-\tilde z, u-\tilde z\rangle\le 0$, hence
\[
\|u-\tilde z\|_2^2
=\|u-z+(z-\tilde z)\|_2^2
=\|u-z\|_2^2+\|z-\tilde z\|_2^2+2\langle u-z, z-\tilde z\rangle
\le \|u-z\|_2^2.
\]
Taking the minimum over $u\in\Theta$ gives the monotonicity under projection:
\begin{equation}\label{eq:C:proj_noninc}
F_{\lambda_{\mathrm{env}}}(\mathrm{proj}_\Theta(z))\le F_{\lambda_{\mathrm{env}}}(z)\quad\text{for all }z\in\mathbb{R}^d.
\end{equation}

Now define the unprojected point $z_t:=\theta_t-\eta G(\theta_t;Z_t)$, so $\theta_{t+1}=\mathrm{proj}_\Theta(z_t)$.
By~\eqref{eq:C:proj_noninc},
\[
F_{\lambda_{\mathrm{env}}}(\theta_{t+1})\le F_{\lambda_{\mathrm{env}}}(z_t).
\]
By Lemma \ref{lem:moreau_properties}, $F_{\lambda_{\mathrm{env}}}$ is $L_{\mathrm{env}}$-smooth, hence
\[
F_{\lambda_{\mathrm{env}}}(\theta_t-\eta G(\theta_t;Z_t))
\le F_{\lambda_{\mathrm{env}}}(\theta_t)
-\eta\langle \nabla F_{\lambda_{\mathrm{env}}}(\theta_t), G(\theta_t;Z_t)\rangle
+\frac{L_{\mathrm{env}}\eta^2}{2}\|G(\theta_t;Z_t)\|_2^2.
\]
Combining with the previous display and taking conditional expectation yields
\[
\mathbb{E}\bigl[F_{\lambda_{\mathrm{env}}}(\theta_{t+1})\mid \theta_t\bigr]
\le F_{\lambda_{\mathrm{env}}}(\theta_t)
-\eta\left\langle \nabla F_{\lambda_{\mathrm{env}}}(\theta_t),\ \mathbb{E}[G(\theta_t;Z_t)\mid \theta_t]\right\rangle
+\frac{L_{\mathrm{env}}\eta^2}{2}\,\mathbb{E}\bigl[\|G(\theta_t;Z_t)\|_2^2\mid \theta_t\bigr].
\]

It remains to bound the inner product. Let $m_t:=\mathbb{E}[G(\theta_t;Z_t)\mid\theta_t]$; we show that
\begin{equation}\label{eq:C:inner_product}
\bigl\langle \nabla F_{\lambda_{\mathrm{env}}}(\theta_t),m_t\bigr\rangle
\ \ge\ (1-\kappa\lambda_{\mathrm{env}})\|\nabla F_{\lambda_{\mathrm{env}}}(\theta_t)\|_2^2 .
\end{equation}
We use the lower model of Lemma~\ref{lem:constructed_oracles}. By that
lemma, $m_t=\nabla L^{\mathrm{SAIL}}(\theta_t)+\lambda\,\mathbb{E}[G^R(\theta_t;Z_t)\mid\theta_t]$, and the
second mean satisfies~\eqref{eq:constructed_R_lower_model} at $\theta=\theta_t$. By
Assumption~\ref{ass:smooth_sail}, for every $u\in\Theta$,
\[
L^{\mathrm{SAIL}}(u)\ \ge\ L^{\mathrm{SAIL}}(\theta_t)
+\langle\nabla L^{\mathrm{SAIL}}(\theta_t),u-\theta_t\rangle
-\frac{L_{\mathrm{SAIL}}}{2}\|u-\theta_t\|_2^2 .
\]
Adding $\lambda$ times~\eqref{eq:constructed_R_lower_model} and using
$L^W_\rho=L^{\mathrm{SAIL}}+\lambda R$ (Theorem~\ref{thm:decomposition_LW}), we obtain, for every
$u\in\Theta$,
\begin{equation}\label{eq:C:mean_lower_model}
L^W_\rho(u)\ \ge\ L^W_\rho(\theta_t)+\langle m_t,u-\theta_t\rangle-\frac{\kappa}{2}\|u-\theta_t\|_2^2 .
\end{equation}

Let $\hat\theta_t:=\prox_{\lambda_{\mathrm{env}}F}(\theta_t)$ and $d_t:=\theta_t-\hat\theta_t$.
Since $F=L^W_\rho+I_\Theta$, the point $\hat\theta_t$ lies in $\Theta$ and minimizes
\[
\phi_t(u):=L^W_\rho(u)+\frac{1}{2\lambda_{\mathrm{env}}}\|u-\theta_t\|_2^2
\]
over $\Theta$; moreover, $d_t=\lambda_{\mathrm{env}}\nabla F_{\lambda_{\mathrm{env}}}(\theta_t)$ by
Lemma~\ref{lem:moreau_properties}. Because $L^W_\rho$ is $\kappa$-weakly convex on $\Theta$, the
function $\phi_t$ is $(1/\lambda_{\mathrm{env}}-\kappa)$-strongly convex on $\Theta$. For $u\in\Theta$
and $s\in(0,1]$, the point $\hat\theta_t+s(u-\hat\theta_t)$ lies in $\Theta$, so
\[
\begin{aligned}
\phi_t(\hat\theta_t)
&\le\phi_t\bigl(\hat\theta_t+s(u-\hat\theta_t)\bigr)\\
&\le(1-s)\,\phi_t(\hat\theta_t)+s\,\phi_t(u)
-\frac{1/\lambda_{\mathrm{env}}-\kappa}{2}\,s(1-s)\|u-\hat\theta_t\|_2^2 .
\end{aligned}
\]
Dividing by $s$ and letting $s\downarrow0$ gives the quadratic growth bound
$\phi_t(u)\ge\phi_t(\hat\theta_t)+\frac12(1/\lambda_{\mathrm{env}}-\kappa)\|u-\hat\theta_t\|_2^2$.
It uses only the segment from $\hat\theta_t$ to $u$, not a first-order condition at $\hat\theta_t$,
and therefore remains valid when $\hat\theta_t$ lies on the boundary of $\Theta$. Taking $u=\theta_t$
yields
\[
L^W_\rho(\theta_t)-L^W_\rho(\hat\theta_t)
\ \ge\ \Bigl(\frac{1}{\lambda_{\mathrm{env}}}-\frac{\kappa}{2}\Bigr)\|d_t\|_2^2 .
\]
On the other hand, \eqref{eq:C:mean_lower_model} with $u=\hat\theta_t$ gives
$L^W_\rho(\theta_t)-L^W_\rho(\hat\theta_t)\le\langle m_t,d_t\rangle+\frac{\kappa}{2}\|d_t\|_2^2$.
Combining the two bounds,
\[
\langle m_t,d_t\rangle\ \ge\ \Bigl(\frac{1}{\lambda_{\mathrm{env}}}-\kappa\Bigr)\|d_t\|_2^2 .
\]
Substituting $d_t=\lambda_{\mathrm{env}}\nabla F_{\lambda_{\mathrm{env}}}(\theta_t)$ and dividing by
$\lambda_{\mathrm{env}}$ gives~\eqref{eq:C:inner_product}; inserting~\eqref{eq:C:inner_product} into
the previous display completes the proof.
\end{proof}

\subsection{Proof of Theorem \ref{thm:moreau_convergence} (convergence rate of the Moreau envelope)}\label{app:C:thm52}

\begin{proof}[Proof of Theorem \ref{thm:moreau_convergence}]
Fix $\lambda_{\mathrm{env}}\in(0,1/\kappa)$ and stepsize $\eta>0$.
By Lemma~\ref{lem:C:one_step}, for each $t$,
\begin{align*}
\mathbb{E}\bigl[F_{\lambda_{\mathrm{env}}}(\theta_{t+1})\mid \theta_t\bigr]
&\le F_{\lambda_{\mathrm{env}}}(\theta_t)
-\eta(1-\kappa\lambda_{\mathrm{env}})\|\nabla F_{\lambda_{\mathrm{env}}}(\theta_t)\|_2^2
+\frac{L_{\mathrm{env}}\eta^2}{2}\,\mathbb{E}\bigl[\|G(\theta_t;Z_t)\|_2^2\mid \theta_t\bigr].
\end{align*}
Take total expectation and sum from $t=0$ to $T-1$ to obtain the telescoping inequality
\begin{align}
\eta(1-\kappa\lambda_{\mathrm{env}})\sum_{t=0}^{T-1}\mathbb{E}\bigl[\|\nabla F_{\lambda_{\mathrm{env}}}(\theta_t)\|_2^2\bigr]
&\le F_{\lambda_{\mathrm{env}}}(\theta_0)-\mathbb{E}\bigl[F_{\lambda_{\mathrm{env}}}(\theta_T)\bigr]
+\frac{L_{\mathrm{env}}\eta^2}{2}\sum_{t=0}^{T-1}\mathbb{E}\bigl[\|G(\theta_t;Z_t)\|_2^2\bigr].\label{eq:C:telescope}
\end{align}
Since $F_{\lambda_{\mathrm{env}}}(\theta)\ge \inf_u F(u)=F_{\inf}$ for all $\theta$, we have
$\mathbb{E}[F_{\lambda_{\mathrm{env}}}(\theta_T)]\ge F_{\inf}$.
By Lemma \ref{lem:Gtot_from_A6_eq22}, $\mathbb{E}\|G(\theta_t;Z_t)\|_2^2\le G_{\mathrm{tot}}^2$ for all $t$.
Thus~\eqref{eq:C:telescope} implies
\[
\eta(1-\kappa\lambda_{\mathrm{env}})\sum_{t=0}^{T-1}\mathbb{E}\bigl[\|\nabla F_{\lambda_{\mathrm{env}}}(\theta_t)\|_2^2\bigr]
\le F_{\lambda_{\mathrm{env}}}(\theta_0)-F_{\inf}+\frac{L_{\mathrm{env}}\eta^2}{2}\,T\,G_{\mathrm{tot}}^2.
\]
Divide by $\eta(1-\kappa\lambda_{\mathrm{env}})T$:
\[
\frac{1}{T}\sum_{t=0}^{T-1}\mathbb{E}\bigl[\|\nabla F_{\lambda_{\mathrm{env}}}(\theta_t)\|_2^2\bigr]
\le
\frac{F_{\lambda_{\mathrm{env}}}(\theta_0)-F_{\inf}}{\eta(1-\kappa\lambda_{\mathrm{env}})T}
+\frac{L_{\mathrm{env}}\eta}{2(1-\kappa\lambda_{\mathrm{env}})}\,G_{\mathrm{tot}}^2.
\]
Let $R$ be uniformly distributed on $\{0,\dots,T-1\}$, independent of the algorithmic randomness.
Then
\[
\mathbb{E}\bigl[\|\nabla F_{\lambda_{\mathrm{env}}}(\theta_R)\|_2^2\bigr]
=\frac{1}{T}\sum_{t=0}^{T-1}\mathbb{E}\bigl[\|\nabla F_{\lambda_{\mathrm{env}}}(\theta_t)\|_2^2\bigr],
\]
and we obtain exactly Eq.\eqref{eq:moreau_rate} of Theorem \ref{thm:moreau_convergence}.
\end{proof}
\begin{proof}[Proof of Corollary \ref{cor:sample_complexity}]

Plug the stated stepsize choice
\[
\eta:=\frac{1}{G_{\mathrm{tot}}}\sqrt{\frac{2\bigl(F_{\lambda_{\mathrm{env}}}(\theta_0)-F_{\inf}\bigr)}{L_{\mathrm{env}}T}}
\]
into the right-hand side of Eq.\eqref{eq:moreau_rate} and simplify to obtain
\[
\mathbb{E}\bigl[\|\nabla F_{\lambda_{\mathrm{env}}}(\theta_R)\|_2^2\bigr]
\le
\frac{G_{\mathrm{tot}}\sqrt{2L_{\mathrm{env}}\bigl(F_{\lambda_{\mathrm{env}}}(\theta_0)-F_{\inf}\bigr)}}{(1-\kappa\lambda_{\mathrm{env}})\sqrt{T}}.
\]
By Lemma \ref{lem:Gtot_from_A6_eq22}, we have 
\[
G_{\mathrm{tot}}^2
\ :=\ 4\left(
G_{\nabla\mathrm{SAIL}}^2 \;+\; \lambda^2\,G_{\partial R}^2 \;+\; \frac{\sigma_{\mathrm{SAIL}}^2+\lambda^2\sigma_R^2}{B}
\right).
\]

To ensure that $\E\!\left[\|\nabla F_{\lambda_{\mathrm{env}}}(\theta_R)\|_2^2\right]\le \varepsilon$, we substitute $G_{\mathrm{tot}}^2$ into Theorem \ref{thm:moreau_convergence} and obtain that it suffices to take
\begin{equation*}
   \begin{aligned}
      T\ &\ge\
\frac{2\,G_{\mathrm{tot}}^2\bigl(F_{\lambda_{\mathrm{env}}}(\theta_0)-F_{\inf}\bigr)}{\lambda_{\mathrm{env}}(1-\kappa\lambda_{\mathrm{env}})^3\,\varepsilon^2} \\
&=\ \frac{8\left(
C + \frac{\sigma_{\mathrm{SAIL}}^2+\lambda^2\sigma_R^2}{B}
\right)\bigl(F_{\lambda_{\mathrm{env}}}(\theta_0)-F_{\inf}\bigr)}{\lambda_{\mathrm{env}}(1-\kappa\lambda_{\mathrm{env}})^3\,\varepsilon^2},
   \end{aligned} 
\end{equation*}
where $C := G_{\nabla\mathrm{SAIL}}^2 + \lambda^2 G_{\partial R}^2$.

\end{proof}

\section{On the practical role of the proximal point}\label{app:prox-practice}
This section provides a practical interpretation of the proximal-point and Moreau-envelope constructions used in our analysis. Throughout, recall the constrained objective
\[
    F(\theta) := L^W_{\rho}(\theta) + I_{\Theta}(\theta),
\]
and the envelope parameter $\lambda_{\rm env}\in(0,1/\kappa)$ from Definition \ref{def:moreau_envelope}, where $\kappa$ is the weak convexity constant from Theorem \ref{thm:weak_convexity_composite}.

The Moreau envelope
\[
    F_{\lambda_{\rm env}}(\theta)\;=\;\min_{u\in\mathbb{R}^d}\Big\{F(u)+\tfrac{1}{2\lambda_{\rm env}}\|u-\theta\|_2^2\Big\}
\]
can be viewed as a \emph{smoothed} surrogate of the potentially nonsmooth, nonconvex objective $F$. The quadratic term discourages large moves away from the current iterate $\theta$, and the minimizer of the inner problem---the \emph{proximal point}---acts as a locally stabilized refinement of $\theta$.
Concretely, for any iterate $\theta$ we write
\[
    \hat\theta := \prox_{\lambda_{\rm env}F}(\theta)\in\arg\min_{u\in\mathbb{R}^d}
    \Big\{F(u)+\tfrac{1}{2\lambda_{\rm env}}\|u-\theta\|_2^2\Big\}.
\]
Because $F$ is $\kappa$-weakly convex and $\lambda_{\rm env}<1/\kappa$, the proximal subproblem is strongly convex in $u$ and hence has a unique minimizer. This is the fundamental reason the envelope is differentiable and why $\|\nabla F_{\lambda_{\rm env}}(\theta)\|_2$ is a meaningful stationarity proxy (Lemma \ref{lem:moreau_properties}).

Algorithm \ref{alg:rscgd} does \emph{not} require solving the proximal subproblem at each iteration. The algorithm updates $\theta_t$ using stochastic subgradient information for $L^W_{\rho}=L_{\mathrm{SAIL}}+\lambda R$ followed by projection onto $\Theta$, and then outputs a random iterate $\theta_R$. The proximal point
\[
    \hat\theta_R := \prox_{\lambda_{\rm env}F}(\theta_R)
\]
is introduced only as a \emph{theoretical device} to translate envelope stationarity at $\theta_R$ into near-stationarity of the original constrained objective at $\hat\theta_R$ (see Lemma \ref{lem:moreau_properties} and Theorem \ref{thm:moreau_convergence}). In particular, Lemma \ref{lem:moreau_properties} implies that $\theta_R$ and $\hat\theta_R$ are close whenever the envelope gradient is small:
\[
    \|\theta_R-\hat\theta_R\|_2 \;=\; \lambda_{\rm env}\,\|\nabla F_{\lambda_{\rm env}}(\theta_R)\|_2,
\]
and also that $\hat\theta_R$ is nearly stationary for $F$ in the sense that
\[
    \dist\bigl(0,\partial F(\hat\theta_R)\bigr)\;\le\;\|\nabla F_{\lambda_{\rm env}}(\theta_R)\|_2.
\]
Thus, driving $\|\nabla F_{\lambda_{\rm env}}(\theta_R)\|_2$ small (as guaranteed in expectation by Theorem \ref{thm:moreau_convergence}, with sample complexity summarized in Corollary \ref{cor:sample_complexity}) simultaneously guarantees proximity to, and near-stationarity of, a point for the original constrained robust objective.


\begin{lemma}[What if we do not compute the proximal point?]\label{lem:no-prox-output}
Assume $F:\mathbb{R}^d\to(-\infty,+\infty]$ is $\kappa$-weakly convex and bounded below, and fix
$\lambda_{\mathrm{env}}\in(0,1/\kappa)$ as in Definition~\ref{def:moreau_envelope}. For any $\theta\in\mathbb{R}^d$, let
\[
\hat{\theta} := \prox_{\lambda_{\mathrm{env}}F}(\theta),
\qquad
\xi := \nabla F_{\lambda_{\mathrm{env}}}(\theta) = \frac{1}{\lambda_{\mathrm{env}}}(\theta-\hat{\theta})
\]
(Definition \ref{def:moreau_envelope} and Lemma \ref{lem:moreau_properties}). Then:
\begin{align}
\|\theta-\hat{\theta}\| &= \lambda_{\mathrm{env}}\|\xi\|, \label{eq:no-prox-gap}\\
\dist\bigl(0,\partial F(\hat{\theta})\bigr) &\le \|\xi\|, \label{eq:no-prox-proxstationary}\\
\dist\bigl(0,\partial F(\theta)\bigr) &\ge (1-\kappa\lambda_{\mathrm{env}})\|\xi\|. \label{eq:no-prox-iterate-lb}
\end{align}
On the other hand, under $\kappa$-weak convexity alone, there is \emph{no} universal constant $C$ such that
\[
\dist\bigl(0,\partial F(\theta)\bigr)\;\le\; C\,\|\nabla F_{\lambda_{\mathrm{env}}}(\theta)\|
\quad\text{for all $\theta$ and all $\kappa$-weakly convex $F$}
\]
(even in the convex case $\kappa=0$).
\end{lemma}

\begin{proof}
We first prove \eqref{eq:no-prox-gap}--\eqref{eq:no-prox-proxstationary}. By Lemma~\ref{lem:moreau_properties} (see also the
first-order optimality condition for the proximal subproblem),
\[
0 \in \partial F(\hat{\theta}) + \frac{1}{\lambda_{\mathrm{env}}}(\hat{\theta}-\theta),
\qquad\text{equivalently}\qquad
\frac{1}{\lambda_{\mathrm{env}}}(\theta-\hat{\theta}) \in \partial F(\hat{\theta}).
\]
With $\xi := \frac{1}{\lambda_{\mathrm{env}}}(\theta-\hat{\theta}) = \nabla F_{\lambda_{\mathrm{env}}}(\theta)$,
taking norms yields $\|\theta-\hat{\theta}\|=\lambda_{\mathrm{env}}\|\xi\|$, which is \eqref{eq:no-prox-gap}.
Moreover, since $\xi\in\partial F(\hat{\theta})$, we have
$\dist\bigl(0,\partial F(\hat{\theta})\bigr)\le \|\xi\|$, which is \eqref{eq:no-prox-proxstationary}.

Next we prove \eqref{eq:no-prox-iterate-lb}. By Lemma~\ref{lem:C:monotonicity} (Monotonicity inequality), for every
$v\in\partial F(\theta)$,
\[
\langle \xi, v\rangle \ge (1-\kappa\lambda_{\mathrm{env}})\|\xi\|^2.
\]
By Cauchy--Schwarz, $\langle \xi,v\rangle \le \|\xi\|\,\|v\|$, hence for every $v\in\partial F(\theta)$,
\[
\|v\| \ge (1-\kappa\lambda_{\mathrm{env}})\|\xi\|.
\]
Taking the infimum over $v\in\partial F(\theta)$ gives \eqref{eq:no-prox-iterate-lb}.

For the stated consequences with the algorithmic output, apply \eqref{eq:no-prox-gap}--\eqref{eq:no-prox-proxstationary}
with $\theta=\theta_R$ and $\hat{\theta}=\hat{\theta}_R$. Then
\[
\|\theta_R-\hat{\theta}_R\|^2=\lambda_{\mathrm{env}}^2\|\nabla F_{\lambda_{\mathrm{env}}}(\theta_R)\|^2,
\qquad
\dist^2\bigl(0,\partial F(\hat{\theta}_R)\bigr)\le \|\nabla F_{\lambda_{\mathrm{env}}}(\theta_R)\|^2.
\]
Taking expectations, using Jensen's inequality for $\mathbb{E}\|\theta_R-\hat{\theta}_R\|$, and
plugging $\mathbb{E}\|\nabla F_{\lambda_{\mathrm{env}}}(\theta_R)\|^2\le\varepsilon$ yields the two bounds.

Finally, we show that no universal constant $C$ can upper bound
$\dist(0,\partial F(\theta))$ by $\|\nabla F_{\lambda_{\mathrm{env}}}(\theta)\|$ under weak convexity alone.
Consider $d=1$ and $\Theta=\mathbb{R}$, and define $F(\theta)=|\theta|$, which is convex (thus $\kappa=0$).
Fix $\lambda_{\mathrm{env}}=1$ and any $\theta\in(0,1)$. The proximal point is
\[
\hat{\theta}
=\arg\min_{u\in\mathbb{R}}\Big\{|u|+\tfrac12(u-\theta)^2\Big\}
=0,
\]
so $\nabla F_{\lambda_{\mathrm{env}}}(\theta)=\theta-\hat{\theta}=\theta$. However, for $\theta>0$,
$\partial F(\theta)=\{1\}$ and therefore $\dist(0,\partial F(\theta))=1$. Hence
\[
\frac{\dist(0,\partial F(\theta))}{\|\nabla F_{\lambda_{\mathrm{env}}}(\theta)\|}
=\frac{1}{\theta}\to\infty \quad\text{as }\theta\downarrow 0,
\]
which rules out any finite universal $C$.
\end{proof}

\begin{remark}[Interpretation for practice]\label{rem:no-prox-interpret}
Even if Algorithm~1 outputs only
$\theta_R$ (without solving the proximal subproblem), Theorem~\ref{thm:moreau_convergence} still certifies that $\theta_R$ lies within
$O(\lambda_{\mathrm{env}}\sqrt{\varepsilon})$ of a point $\hat{\theta}_R$ that is $O(\sqrt{\varepsilon})$-stationary
for the original constrained objective $F$. In contrast, without additional regularity beyond $\kappa$-weak convexity,
one cannot generally convert envelope stationarity into a bound on $\dist(0,\partial F(\theta_R))$ itself; this is why
the analysis (and the stationarity notion) is phrased in terms of the proximal point.
\end{remark}

\subsection{Practical computation of the proximal point}\label{app:prox-computation}

In the analysis we introduce the proximal point
\begin{equation}\label{eq:prox-subproblem-post}
\hat{\theta}_R \;:=\; \prox_{\lambda_{\rm env}}F(\theta_R)
\;=\;\arg\min_{\theta\in\Theta}\Big\{L_\rho^W(\theta)+\frac{1}{2\lambda_{\rm env}}\|\theta-\theta_R\|_2^2\Big\},
\end{equation}
where \(F(\theta):=L_\rho^W(\theta)+I_\Theta(\theta)\) and \(\lambda_{\rm env}\in(0,1/\kappa)\).
By weak convexity of \(F\) and \(\lambda_{\rm env}<1/\kappa\), the objective in~\eqref{eq:prox-subproblem-post}
is strongly convex and hence admits a unique minimizer \(\hat{\theta}_R\).
Although Algorithm~1 outputs \(\theta_R\), one can compute an approximation
\(\bar{\theta}_R\approx\hat{\theta}_R\) as an optional post-processing step, warm-started at \(\theta_R\). A simple inner-loop solver is provided below.

\paragraph{A simple inner-loop solver.}
Let \(\theta^{(0)}:=\theta_R\). For \(k=0,1,\ldots,K-1\), sample \(Z^{(k)}\sim d_{\theta^{(k)}}\) and compute
the stochastic direction
\[
G^{\rm prox}\big(\theta^{(k)};Z^{(k)}\big)
\;:=\; G_{\rm SAIL}\big(\theta^{(k)};Z^{(k)}\big)
\;+\;\lambda\,G_R\big(\theta^{(k)};Z^{(k)}\big)
\;+\;\frac{1}{\lambda_{\rm env}}\big(\theta^{(k)}-\theta_R\big),
\]
which is the usual oracle from Lemma~\ref{lem:constructed_oracles} plus the deterministic gradient of the proximal quadratic.
Then take a projected step
\begin{equation}\label{eq:prox-inner-update}
\theta^{(k+1)} \;:=\; \Pi_\Theta\!\left(\theta^{(k)}-\alpha_k\,G^{\rm prox}\big(\theta^{(k)};Z^{(k)}\big)\right),
\qquad \bar{\theta}_R:=\theta^{(K)}.
\end{equation}
Since the proximal quadratic improves conditioning, this inner loop is typically stable in practice.
Moreover, the theory only requires \(\bar{\theta}_R\) to be a sufficiently accurate approximation;
exact solves of~\eqref{eq:prox-subproblem-post} are not necessary. This is consistent with standard practice in inexact proximal methods, where the proximal subproblem is solved
only approximately while maintaining meaningful convergence guarantees (e.g., \citealt{schmidt2011convergence}).

\begin{lemma}[Stationarity of an inexact proximal point]\label{lem:inexact-prox-stationarity}
Fix \(\lambda_{\rm env}\in(0,1/\kappa)\) and an index \(R\).
Let \(\hat{\theta}_R=\prox_{\lambda_{\rm env}}F(\theta_R)\) and let \(\bar{\theta}_R\in\Theta\) be any point.
Define the proximal objective
\[
\Psi_R(\theta)\;:=\;F(\theta)+\frac{1}{2\lambda_{\rm env}}\|\theta-\theta_R\|_2^2.
\]
If \(\bar{\theta}_R\) satisfies the (first-order) proximal residual bound
\begin{equation}\label{eq:prox-residual}
\dist\!\left(0,\partial \Psi_R(\bar{\theta}_R)\right)\;\le\;\varepsilon_{\rm prox},
\end{equation}
then
\begin{equation}\label{eq:dist-subdiff-bound}
\dist\!\left(0,\partial F(\bar{\theta}_R)\right)
\;\le\;
\|\nabla F_{\lambda_{\rm env}}(\theta_R)\|_2
\;+\;\varepsilon_{\rm prox}
\;+\;\frac{1}{\lambda_{\rm env}}\|\bar{\theta}_R-\hat{\theta}_R\|_2.
\end{equation}
In particular, if one ensures \(\|\bar{\theta}_R-\hat{\theta}_R\|_2\le \lambda_{\rm env}\varepsilon_{\rm geom}\),
then
\(
\dist\!\left(0,\partial F(\bar{\theta}_R)\right)
\le
\|\nabla F_{\lambda_{\rm env}}(\theta_R)\|_2+\varepsilon_{\rm prox}+\varepsilon_{\rm geom}.
\)
\end{lemma}

\begin{proof}
By subdifferential calculus,
\(
\partial \Psi_R(\theta)=\partial F(\theta)+\frac{1}{\lambda_{\rm env}}(\theta-\theta_R).
\)
The residual condition~\eqref{eq:prox-residual} implies there exists
\(g\in\partial F(\bar{\theta}_R)\) such that
\[
\left\|g+\frac{1}{\lambda_{\rm env}}(\bar{\theta}_R-\theta_R)\right\|_2\;\le\;\varepsilon_{\rm prox}.
\]
Therefore,
\[
\dist\!\left(0,\partial F(\bar{\theta}_R)\right)
\;\le\;\|g\|_2
\;\le\;\varepsilon_{\rm prox}+\frac{1}{\lambda_{\rm env}}\|\bar{\theta}_R-\theta_R\|_2.
\]
Next, by the triangle inequality,
\(
\|\bar{\theta}_R-\theta_R\|_2\le \|\bar{\theta}_R-\hat{\theta}_R\|_2+\|\hat{\theta}_R-\theta_R\|_2.
\)
By Lemma~\ref{lem:moreau_properties} (property \(\nabla F_{\lambda_{\rm env}}(\theta)=\frac{1}{\lambda_{\rm env}}(\theta-\prox_{\lambda_{\rm env}}F(\theta))\)),
we have \(\|\hat{\theta}_R-\theta_R\|_2=\lambda_{\rm env}\|\nabla F_{\lambda_{\rm env}}(\theta_R)\|_2\).
Substituting yields~\eqref{eq:dist-subdiff-bound}.
\end{proof}

\begin{remark}[Practical stopping criteria]
    Lemma~\ref{lem:inexact-prox-stationarity} suggests two inexpensive targets for the inner loop:
(i) keep the proximal residual \(\dist(0,\partial\Psi_R(\theta^{(k)}))\) small (as estimated by a minibatch),
and (ii) exploit warm-starting so that \(\|\bar{\theta}_R-\hat{\theta}_R\|_2\) remains modest.

\end{remark}

%% file: appendixD.tex
\section{Finite-response experiment details}
\label{app:experiments}
\label{app:finite-response}

\makeatletter
\@addtoreset{figure}{section}
\@addtoreset{table}{section}
\makeatother
\renewcommand{\thefigure}{\thesection.\arabic{figure}}
\renewcommand{\thetable}{\thesection.\arabic{table}}
\setcounter{figure}{0}
\setcounter{table}{0}

This appendix gives the finite-response experimental details supporting Section~\ref{sec:experiments}. The goal is to document the controlled oracle-stress construction behind the main-text claim: Robust SAIL (ours) sacrifices ideal-oracle fitting in order to reduce oracle sensitivity and improve worst-case oracle stress loss. The code for these experiments is available at
\url{https://github.com/anonymousneurips2026/oracle-robust-online-alignment}.
\subsection{Controlled finite-response setup and nuisance-bias construction}
\label{app:finite-setup}

For each prompt $x$, we fix a candidate response set
\[
\mathcal{Y}_x=\{y_1,\ldots,y_K\},
\qquad
K=12.
\]
The learned policy is a log-linear distribution over this fixed candidate set with feature dimension $d=7$:
\[
\pi_\theta(y_i\mid x)
=
\frac{\exp(\theta^\top\psi(x,y_i))}
{\sum_{j=1}^K\exp(\theta^\top\psi(x,y_j))}.
\]
The reference policy $\pi_{\mathrm{ref}}$ is fixed and induced by frozen reference logits. The policy-induced comparison distribution is
\[
d_\theta(x,y_i,y_j)
=
\mu(x)\pi_\theta(y_i\mid x)\pi_\theta(y_j\mid x).
\]

The ideal oracle is denoted by $P^{\mathrm{id}}$. For $z=(x,y_i,y_j)$, write
\[
p_{\mathrm{id}}(z)
=
P^{\mathrm{id}}(y_i\succ y_j\mid x).
\]
The ideal oracle follows a Bradley--Terry model induced by a synthetic true reward $r^\star$:
\[
p_{\mathrm{id}}(x,i,j)
=
\sigma\!\left(
\frac{r^\star(x,y_i)-r^\star(x,y_j)}{\tau}
\right).
\]
The observed training oracle is a bounded perturbation of this ideal oracle:
\[
p_{\mathrm{obs}}(z)
=
p_{\mathrm{id}}(z)+\rho_{\mathrm{mis}}b(z),
\qquad
\rho_{\mathrm{mis}}\in\{0.05,0.075,0.10\}.
\]
Here $\rho_{\mathrm{mis}}$ is the environment misspecification radius in this controlled benchmark; it controls the distance between the observed training oracle and the ideal oracle. The algorithmic robust radius $\rho_{\mathrm{alg}}$ is a training hyperparameter. We use $q=\rho_{\mathrm{alg}}/\rho_{\mathrm{mis}}$ only because this controlled experiment knows $\rho_{\mathrm{mis}}$.

The perturbation direction $b$ is a structured nuisance bias. Let $s(x,y_i)$ be a redundancy-style scalar attribute given by a fixed projection of $\psi(x,y_i)$ that is orthogonal to the true-reward direction. For a pair $(y_i,y_j)$, define
\[
\Delta s_{ij}(x)=s(x,y_i)-s(x,y_j),
\qquad
\Delta r_{ij}(x)=r^\star(x,y_i)-r^\star(x,y_j),
\]
and the per-prompt normalization
\[
M_x
=
\max_{u,v}|s(x,y_u)-s(x,y_v)|+\varepsilon_{\mathrm{norm}}.
\]
The nuisance-bias direction is
\[
b_{ij}(x)
=
\operatorname{sign}(\Delta s_{ij}(x))
\frac{|\Delta s_{ij}(x)|}{M_x}
\mathbf{1}\{\Delta s_{ij}(x)\Delta r_{ij}(x)\le 0\}.
\]
By construction,
\[
|b_{ij}(x)|\le 1,
\qquad
b_{ji}(x)=-b_{ij}(x).
\]
Thus the observed oracle is biased toward the nuisance attribute only when that attribute conflicts with the true reward. This creates structured oracle misspecification rather than random label noise.

All experiments use $10$ paired seeds for every $\rho_{\mathrm{mis}}$. Evaluation uses exact finite-response expectations, so the reported losses have no Monte Carlo evaluation noise. The candidate pool is fixed before optimization. The frozen checkpoint \texttt{Qwen/Qwen2.5-0.5B-Instruct} \citep{yang2025qwen25} is used only to instantiate language-like finite candidate sets and reference logits; it is not fine-tuned or evaluated as the object of study.

\subsection{Training objectives and hyperparameter grid}
\label{app:training-grid}

The pairwise log-ratio score is
\[
h_\theta(x,y_i,y_j)
=
\log\frac{\pi_\theta(y_i\mid x)}{\pi_{\mathrm{ref}}(y_i\mid x)}
-
\log\frac{\pi_\theta(y_j\mid x)}{\pi_{\mathrm{ref}}(y_j\mid x)}.
\]
The sensitivity penalty is
\[
R(\theta)
=
\mathbb{E}_{z\sim d_\theta}\bigl[|h_\theta(z)|\bigr].
\]
Nominal SAIL is trained under the observed oracle:
\[
\theta_{\mathrm{Nominal}}
\in
\arg\min_\theta L_{\mathrm{SAIL}}(\theta;P^{\mathrm{obs}}).
\]
For $q>0$, Robust SAIL trains
\[
\theta_q
\in
\arg\min_\theta
\left[
L_{\mathrm{SAIL}}(\theta;P^{\mathrm{obs}})
+
\rho_{\mathrm{alg}}\beta R(\theta)
\right],
\qquad
\rho_{\mathrm{alg}}=q\rho_{\mathrm{mis}},
\]
with $\beta=0.1$. The baseline $q=0$ is Nominal SAIL, while $q>0$ are Robust SAIL candidates.

For $\rho_{\mathrm{mis}}\in\{0.05,0.10\}$, the trained grid is
\[
q\in\{0,0.025,0.05,0.10,0.25,0.50,1.00\}.
\]
For the central radius $\rho_{\mathrm{mis}}=0.075$, we also train smaller radii:
\[
q\in\{0,0.005,0.01,0.025,0.05,0.10,0.25,0.50,1.00\}.
\]
The robust operating point $q^\star$ is selected globally for each $\rho_{\mathrm{mis}}$ using the validation split rather than the test split, with seed-pooled validation high-stress AUC rather than per-seed selection.

\subsection{True-to-worst oracle stress metric}
\label{app:true-to-worst-metric}

For each trained policy $\theta$, define the two label-conditional losses
\[
\ell^1_\theta(z)=-\log\sigma(\beta h_\theta(z)),
\qquad
\ell^0_\theta(z)=-\log\sigma(-\beta h_\theta(z)).
\]
The stress evaluation uncertainty interval is centered at the observed oracle and uses the environment evaluation radius $\rho_{\mathrm{mis}}$:
\[
p_{\mathrm{low}}(z)=\max\{0,p_{\mathrm{obs}}(z)-\rho_{\mathrm{mis}}\},
\qquad
p_{\mathrm{high}}(z)=\min\{1,p_{\mathrm{obs}}(z)+\rho_{\mathrm{mis}}\}.
\]
This radius is not $\rho_{\mathrm{alg}}$: $\rho_{\mathrm{alg}}$ controls training, while $\rho_{\mathrm{mis}}$ controls the stress evaluation.

The policy-specific worst-case oracle is
\[
p^\theta_{\mathrm{wc}}(z)
=
\begin{cases}
p_{\mathrm{high}}(z), & \ell^1_\theta(z)\ge \ell^0_\theta(z),\\
p_{\mathrm{low}}(z), & \ell^1_\theta(z)<\ell^0_\theta(z).
\end{cases}
\]
The true-to-worst oracle path is
\[
p^\theta_\alpha(z)
=
(1-\alpha)p_{\mathrm{id}}(z)
+
\alpha p^\theta_{\mathrm{wc}}(z),
\qquad
\alpha\in[0,1].
\]
Thus $\alpha=0$ is the ideal oracle, $\alpha=1$ is the policy-specific worst-case oracle, and larger $\alpha$ means stronger adversarial oracle stress. The stress loss is
\[
L_{\mathrm{stress}}(\theta;\alpha)
=
\mathbb{E}_{z\sim d_\theta}
\left[
p^\theta_\alpha(z)\ell^1_\theta(z)
+
\bigl(1-p^\theta_\alpha(z)\bigr)\ell^0_\theta(z)
\right].
\]
We report
\[
L_{\mathrm{id}}(\theta)=L_{\mathrm{stress}}(\theta;0),
\qquad
L_{\mathrm{wc}}(\theta)=L_{\mathrm{stress}}(\theta;1),
\]
and
\[
\mathrm{AUC}_{0.5:1}(\theta)
=
\frac{1}{6}
\sum_{\alpha\in\{0.5,0.6,0.7,0.8,0.9,1.0\}}
L_{\mathrm{stress}}(\theta;\alpha).
\]
All loss-style metrics are lower-is-better.

For fixed $\theta$, the stress loss is affine in $\alpha$:
\[
L_{\mathrm{stress}}(\theta;\alpha)
=
(1-\alpha)L_{\mathrm{id}}(\theta)+\alpha L_{\mathrm{wc}}(\theta).
\]
Since $p_{\mathrm{obs}}(z)=p_{\mathrm{id}}(z)+\rho_{\mathrm{mis}}b(z)$ and the worst-case endpoint around $p_{\mathrm{obs}}$ uses radius $\rho_{\mathrm{mis}}$, the unclipped endpoint decomposes as
\[
L_{\mathrm{wc}}(\theta)
=
L_{\mathrm{id}}(\theta)
+
\rho_{\mathrm{mis}}
\mathbb{E}_{z\sim d_\theta}
\left[
b(z)\bigl(\ell^1_\theta(z)-\ell^0_\theta(z)\bigr)
\right]
+
\rho_{\mathrm{mis}}\beta R(\theta).
\]
This separates ideal-oracle fitting, the fixed observed-oracle bias direction, and the worst-case sensitivity term. In all reported runs the clipped mass is zero, so this unclipped decomposition is exact. The mean interval widths are $2\rho_{\mathrm{mis}}=0.10,0.15,0.20$ for $\rho_{\mathrm{mis}}=0.05,0.075,0.10$, respectively.

\subsection{Main all-radius stress results}
\label{app:all-radius-results}

Table~\ref{tab:app-stress-deltas} reports paired test deltas between Robust SAIL (ours) at $q^\star$ and Nominal SAIL. Robust SAIL (ours) has higher $L_{\mathrm{id}}$, which is the intended ideal-oracle fitting sacrifice. It simultaneously reduces $R(\theta)$ by roughly two orders of magnitude and improves both $L_{\mathrm{wc}}$ and $\mathrm{AUC}_{0.5:1}$ on all $10/10$ paired seeds for every misspecification radius.

\begin{table}[t]
\centering
\small
\caption{Paired deltas, Robust SAIL (ours, $q^\star$) minus Nominal SAIL, on the test split. Losses and $R(\theta)$ are lower-is-better, so negative deltas indicate improvement for Robust SAIL.}
\label{tab:app-stress-deltas}
\resizebox{\linewidth}{!}{
\begin{tabular}{c c r r r r c c}
\toprule
$\rho_{\mathrm{mis}}$
& $q^\star$
& $\Delta L_{\mathrm{id}}$
& $\Delta R$
& $\Delta L_{\mathrm{wc}}$
& $\Delta\mathrm{AUC}_{0.5:1}$
& crossing $\alpha$
& wins \\
\midrule
0.05 & 1.0 & $+0.0011$ & $-0.4931$ & $-0.0016$ & $-0.0009$ & 0.417 & 10/10 \\
0.075 & 1.0 & $+0.0008$ & $-0.4904$ & $-0.0028$ & $-0.0019$ & 0.226 & 10/10 \\
0.10 & 1.0 & $+0.0004$ & $-0.5144$ & $-0.0042$ & $-0.0030$ & 0.092 & 10/10 \\
\bottomrule
\end{tabular}
}
\end{table}

The crossing point moves earlier as $\rho_{\mathrm{mis}}$ increases: under larger oracle stress, the reduction in sensitivity dominates the ideal-oracle fitting cost sooner along the true-to-worst path.

\begin{figure}[t]
\centering
\includegraphics[width=\linewidth]{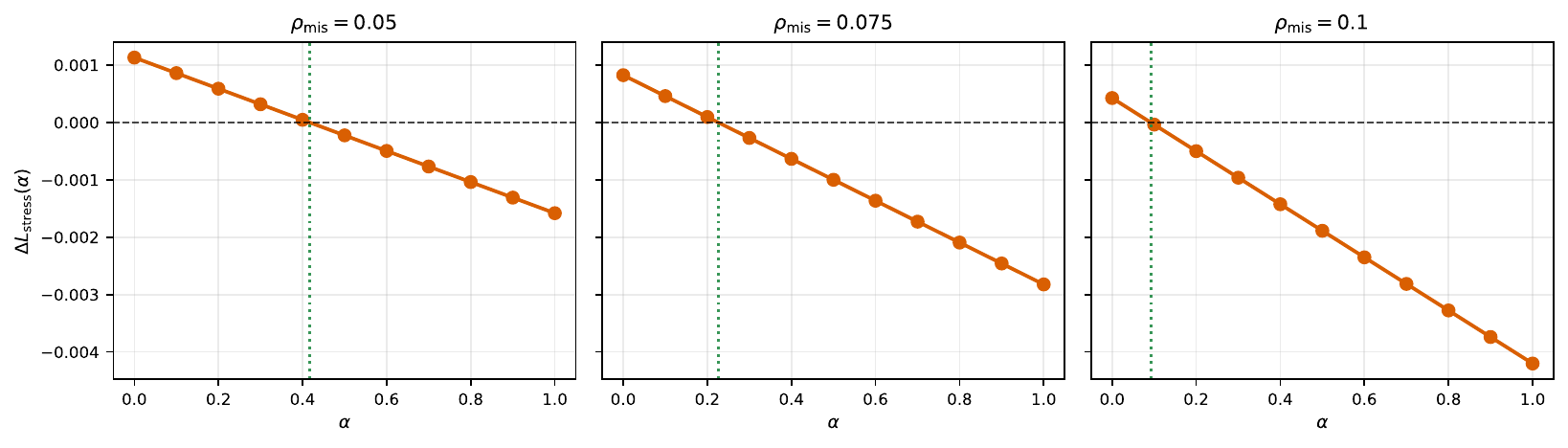}
\caption{Paired true-to-worst stress gaps across misspecification radii. The plotted quantity is $\Delta L_{\mathrm{stress}}(\alpha)=L_{\mathrm{stress}}(\theta_{\mathrm{Robust}};\alpha)-L_{\mathrm{stress}}(\theta_{\mathrm{Nominal}};\alpha)$. Negative values indicate lower stress loss for Robust SAIL (ours). Robust SAIL is worse near the ideal oracle but becomes better as the evaluation oracle moves toward the worst-case endpoint.}
\label{fig:app-ttw-all-rho}
\end{figure}

\subsection{Radius selection and calibration}
\label{app:radius-selection}

For each $\rho_{\mathrm{mis}}$, the robust radius is selected by seed-pooled validation high-stress AUC:
\[
q^\star(\rho_{\mathrm{mis}})
=
\arg\min_{q>0}
\frac{1}{|\mathcal{S}|}
\sum_{s\in\mathcal{S}}
\mathrm{AUC}^{\mathrm{val}}_{0.5:1}(\theta_{q,s}).
\]
The nominal baseline $q=0$ is excluded from the robust candidate pool. Selection is global per $\rho_{\mathrm{mis}}$, not per seed, and the test split is used only after $q^\star$ is fixed. The same selected value is obtained if validation $L_{\mathrm{wc}}$ is used instead of validation AUC.

\begin{table}[t]
\centering
\small
\caption{Validation selection of $q^\star$. Both validation high-stress AUC and validation $L_{\mathrm{wc}}$ select the full-radius robust objective for all three radii.}
\label{tab:app-q-selection}
\begin{tabular}{c c c c c c}
\toprule
$\rho_{\mathrm{mis}}$
& $q^\star$ by AUC
& $q^\star$ by $L_{\mathrm{wc}}$
& val $\mathrm{AUC}_{0.5:1}$
& val $L_{\mathrm{wc}}$
& val $R$ \\
\midrule
0.05 & 1.0 & 1.0 & 0.6931 & 0.6932 & 0.0038 \\
0.075 & 1.0 & 1.0 & 0.6932 & 0.6932 & 0.0025 \\
0.10 & 1.0 & 1.0 & 0.6932 & 0.6932 & 0.0026 \\
\bottomrule
\end{tabular}
\end{table}

The selected value has a precise controlled-benchmark interpretation:
\[
q^\star=1
\quad\Longleftrightarrow\quad
\rho_{\mathrm{alg}}=\rho_{\mathrm{mis}}.
\]
This is a radius-calibration result, not a deployment-time rule. In real applications, $\rho_{\mathrm{mis}}$ is unknown. One should tune $\rho_{\mathrm{alg}}$ directly, or equivalently $\lambda=\rho_{\mathrm{alg}}\beta$, using one of the following criteria:  robust validation performance, uncertainty calibration from annotator disagreement or a domain-specific risk tolerance.

\begin{figure}[t]
\centering
\includegraphics[width=0.92\linewidth]{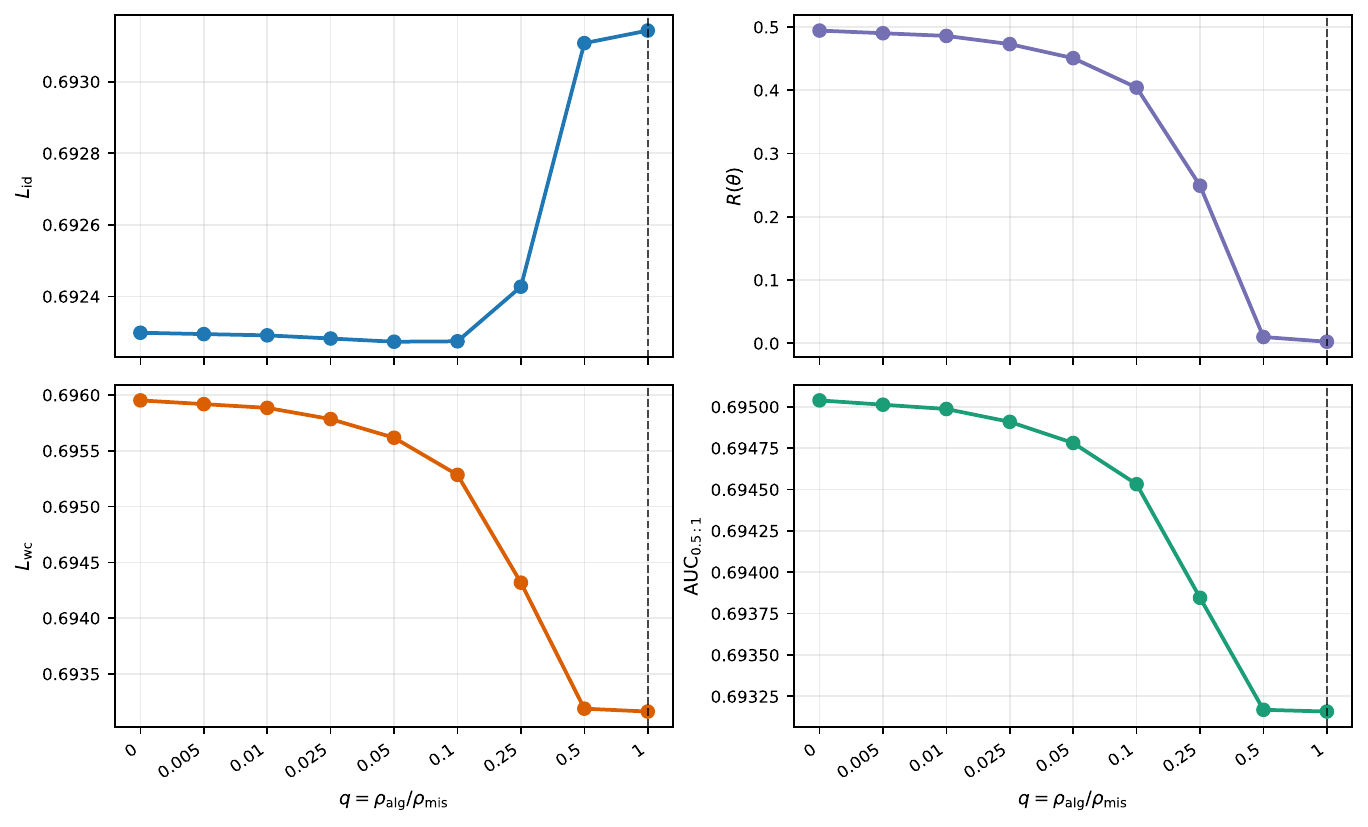}
\caption{Radius-selection sweep at $\rho_{\mathrm{mis}}=0.075$. Increasing $q=\rho_{\mathrm{alg}}/\rho_{\mathrm{mis}}$ reduces the sensitivity penalty and the high-stress oracle loss, while increasing the ideal-oracle loss. Validation high-stress AUC selects the full-radius robust objective $q^\star=1$.}
\label{fig:app-q-sweep}
\end{figure}

\subsection{Utility-preserving operating point selection}
\label{app:utility-selection}

The main experiment selects $q^\star$ using a robust-first validation criterion. If a reader instead wants to preserve ideal reward, the same grid can be used to choose a smaller $\rho_{\mathrm{alg}}$ under a utility constraint. The quantity $E[r^\star]$ is available only in this controlled benchmark and is used here only as a diagnostic.

For a utility tolerance $\epsilon_{\mathrm{util}}$, define
\[
q_{\mathrm{util}}
=
\max\left\{
q:
E[r^\star(\theta_q)]
\ge
E[r^\star(\theta_0)]-\epsilon_{\mathrm{util}}
\right\}.
\]
At the central radius $\rho_{\mathrm{mis}}=0.075$, the utility values are:
\[
\begin{array}{c|ccccccccc}
q & 0 & 0.005 & 0.01 & 0.025 & 0.05 & 0.10 & 0.25 & 0.50 & 1.00\\
\hline
E[r^\star] & 0.2190 & 0.2188 & 0.2186 & 0.2181 & 0.2172 & 0.2154 & 0.2098 & 0.2011 & 0.2008
\end{array}
\]
For $\epsilon_{\mathrm{util}}=10^{-3}$, this rule selects $q_{\mathrm{util}}=0.025$. Thus $q^\star=1$ is best under the robust-first high-stress AUC criterion, while $q_{\mathrm{util}}=0.025$ is a much lighter robust setting that preserves true utility within the specified tolerance. Radius selection is therefore a risk-tolerance problem, not something the robust objective solves automatically.

\subsection{KL-to-reference smoothing baseline}
\label{app:kl-baseline}

We also evaluate KL-to-reference smoothing:
\[
\min_\theta
L_{\mathrm{SAIL}}(\theta;P^{\mathrm{obs}})
+
\eta_{\mathrm{KL}}
\mathbb{E}_{x}
D_{\mathrm{KL}}
\left(
\pi_\theta(\cdot\mid x)
\Vert
\pi_{\mathrm{ref}}(\cdot\mid x)
\right).
\]
The hyperparameter $\eta_{\mathrm{KL}}$ is a generic smoothing weight. It is not an oracle uncertainty radius and does not correspond to $\rho_{\mathrm{mis}}$.

KL-to-reference smoothing is selected by validation high-stress AUC over its own grid. At $\rho_{\mathrm{mis}}=0.075$, the selected value is $\eta_{\mathrm{KL}}=0.1$. Table~\ref{tab:app-kl} shows that tuned KL smoothing can produce a stress profile close to Robust SAIL (ours) in this finite-response benchmark. Robust SAIL differs in that its penalty is derived from the worst-case oracle objective and directly controls the explicit sensitivity term $R(\theta)$.

\begin{table}[t]
\centering
\small
\caption{KL-to-reference smoothing baseline at $\rho_{\mathrm{mis}}=0.075$. Lower is better for all columns.}
\label{tab:app-kl}
\begin{tabular}{l c c c c}
\toprule
Method & $L_{\mathrm{id}}$ & $R(\theta)$ & $L_{\mathrm{wc}}$ & $\mathrm{AUC}_{0.5:1}$ \\
\midrule
Nominal SAIL & 0.6923 & 0.4928 & 0.6960 & 0.6951 \\
Robust SAIL (ours) ($q^\star=1$) & 0.6931 & 0.0025 & 0.6932 & 0.6932 \\
KL-to-reference ($\eta_{\mathrm{KL}}=0.1$) & 0.6930 & 0.0385 & 0.6933 & 0.6932 \\
\bottomrule
\end{tabular}
\end{table}

\begin{figure}[t]
\centering
\includegraphics[width=0.72\linewidth]{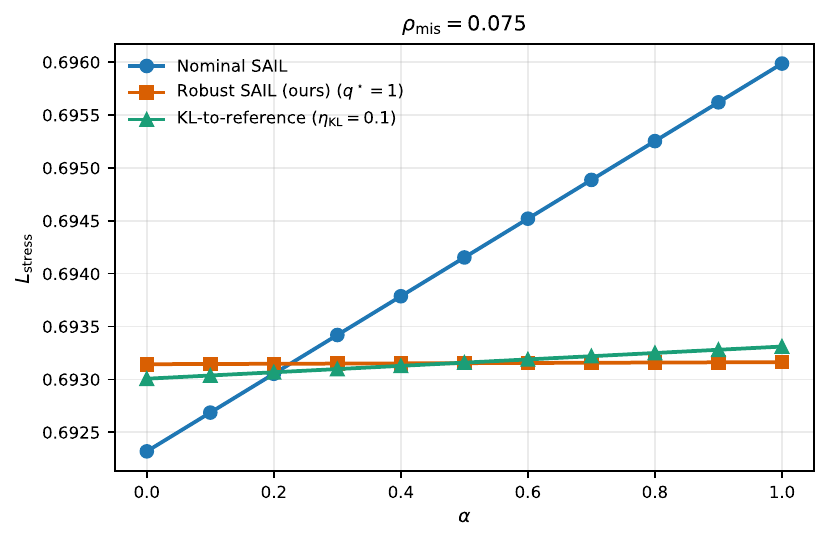}
\caption{KL-to-reference smoothing baseline at $\rho_{\mathrm{mis}}=0.075$. A tuned KL baseline also flattens the stress curve, showing that strong reference smoothing can approximate the robust operating point in this finite-response benchmark. Robust SAIL (ours) differs in that its penalty is derived from the worst-case oracle objective and directly controls the sensitivity term $R(\theta)$.}
\label{fig:app-kl-curves}
\end{figure}

\subsection{Shared-adversary common-worst diagnostic}
\label{app:common-worst}

The main stress metric uses a policy-specific worst-case oracle, which is the appropriate robust-risk evaluation. As a sanity check, we also construct a shared common adversary using the Nominal policy:
\[
p_{\mathrm{wc}}^{\theta_{\mathrm{Nominal}}}(z).
\]
Both methods are then evaluated along the same interpolation path
\[
p_{\alpha}^{\mathrm{common}}(z)
=
(1-\alpha)p_{\mathrm{id}}(z)
+
\alpha p_{\mathrm{wc}}^{\theta_{\mathrm{Nominal}}}(z).
\]
The common-adversary stress loss is
\[
L_{\mathrm{stress}}^{\mathrm{common}}(\theta;\alpha)
=
\mathbb{E}_{z\sim d_\theta}
\left[
p_{\alpha}^{\mathrm{common}}(z)\ell^1_\theta(z)
+
\bigl(1-p_{\alpha}^{\mathrm{common}}(z)\bigr)\ell^0_\theta(z)
\right].
\]
We plot the paired gap
\[
\Delta L_{\mathrm{stress}}^{\mathrm{common}}(\alpha)
=
L_{\mathrm{stress}}^{\mathrm{common}}(\theta_{\mathrm{Robust}};\alpha)
-
L_{\mathrm{stress}}^{\mathrm{common}}(\theta_{\mathrm{Nominal}};\alpha).
\]

\begin{figure}[t]
\centering
\includegraphics[width=\linewidth]{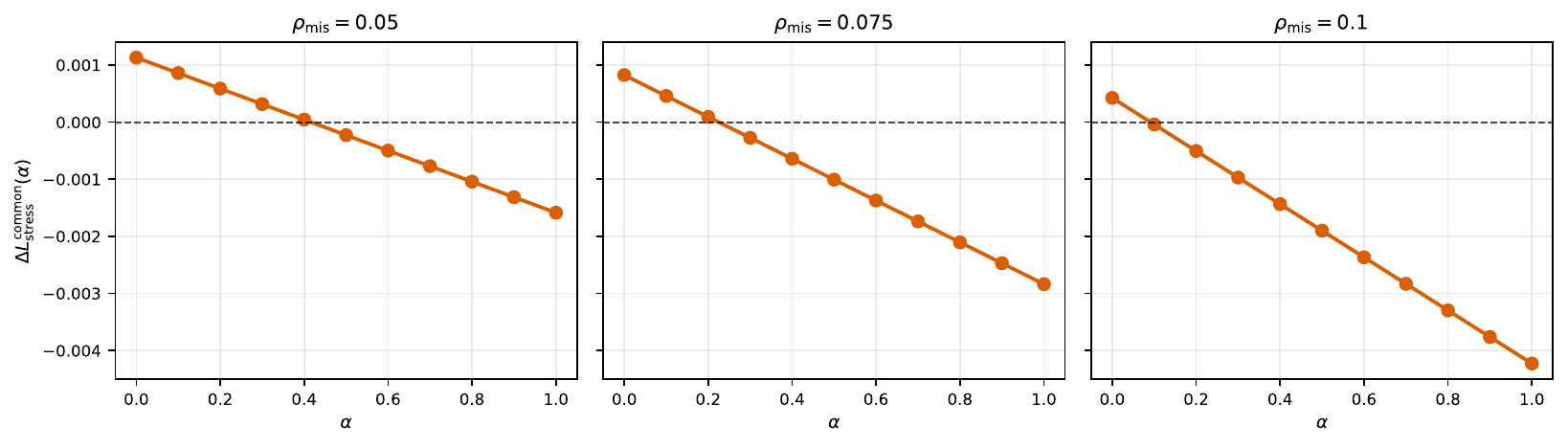}
\caption{Shared-adversary common-worst diagnostic. Negative values mean Robust SAIL (ours) has lower stress loss. Robust SAIL still becomes better under high stress when both methods are evaluated against the same Nominal-policy adversary, confirming that the main result is not an artifact of using different endpoint oracles for different policies.}
\label{fig:app-common-worst}
\end{figure}

\subsection{Statistical reporting and reproducibility}
\label{app:stats-repro}

All reported averages are over $10$ paired seeds. Error bars, where shown, are standard errors over paired seeds unless otherwise stated. Pairwise comparisons use seed-level paired differences. For all main stress results, clipped mass is zero. All reported losses use exact finite-response sums. Validation selection is performed before test reporting.

All experiments are finite-response experiments over fixed candidate sets. Each method-radius-seed run is short; in our implementation each cell trains in less than one minute on a single GPU. The full sweep was run on an internal workstation with four NVIDIA GeForce RTX 2080 Ti GPUs with 11 GB memory per GPU. The controlled true reward $E[r^\star]$ is reported only as an appendix diagnostic and is never used for selecting the robust-first $q^\star$.

\section{Limitations and future work}\label{app:limitations}

Our results apply to a finite-response, log-linear abstraction of online alignment. Assumptions~\ref{ass:loglinear_sail}, \ref{ass:finiteY_boundedpsi}, and \ref{ass:bounded_theta} require a softmax policy over a finite candidate set, bounded features, and a bounded convex parameter set. These assumptions enable explicit weak-convexity and stochastic-oracle bounds, but they do not directly cover transformer-scale autoregressive policies over all token sequences. The decomposition of Theorem~\ref{thm:decomposition_LW} itself holds for any policy with well-defined likelihood ratios with respect to the reference policy; what does not carry over is the subsequent analysis. For an autoregressive neural policy, $h_\theta$ is nonlinear in $\theta$ and the score involves full-sequence gradients, so the bounded-feature arguments no longer yield uniform curvature or moment constants. Extending the analysis beyond finite log-linear policies, for instance under local curvature and sequence-gradient moment conditions, is an important direction for future work.

Our convergence guarantee is stated through the Moreau proximal point of the returned iterate rather than the iterate itself. This is a limitation, but not a purely surrogate one: standard envelope theory ensures the proximal point is close to the returned iterate and nearly stationary for the original constrained objective.

\section{Broader impacts}

This work may have positive societal impact by improving the robustness of online preference-based alignment when the available preference oracle is misspecified. In particular, explicitly modeling oracle uncertainty can help reduce over-optimization of annotator, reward-model, or feedback-loop quirks, and can make the assumptions behind robustness guarantees more transparent. At the same time, robust alignment objectives do not by themselves guarantee safe deployment: if the uncertainty set is too narrow, miscentered, or based on biased feedback, the resulting policy may still preserve or amplify undesirable preferences. Improved alignment optimization tools could also be misused to optimize systems toward manipulative, biased, or otherwise harmful preference oracles. Our empirical study is limited to controlled finite-response experiments and does not release a deployable language model or dataset. We therefore emphasize explicit assumptions, validation-based radius selection, and transparency about the scope of the robustness guarantee as partial mitigations.

%% file: checklist.tex
\section*{NeurIPS Paper Checklist}

\begin{enumerate}

\item {\bf Claims}
    \item[] Question: Do the main claims made in the abstract and introduction accurately reflect the paper's contributions and scope?
    \item[] Answer: \answerYes{} 
    \item[] Justification: These claims are supported by Theorems~\ref{thm:decomposition_LW}, \ref{thm:weak_convexity_composite}, \ref{thm:moreau_convergence} and Corollary~\ref{cor:sample_complexity}, and the experiments in Section~\ref{sec:experiments} and Appendix~\ref{app:experiments}.

\item {\bf Limitations}
    \item[] Question: Does the paper discuss the limitations of the work performed by the authors?
    \item[] Answer: \answerYes{} 
    \item[] Justification: Please see Appendix~\ref{app:limitations}.

\item {\bf Theory assumptions and proofs}
    \item[] Question: For each theoretical result, does the paper provide the full set of assumptions and a complete (and correct) proof?
    \item[] Answer: \answerYes{} 
    \item[] Justification: The assumptions are stated in Assumptions~\ref{ass:nondegenerate_oracle}--\ref{ass:lower_bounded} and Lemma~\ref{lem:constructed_oracles}, and complete proofs are given in Appendices~\ref{app:A}--\ref{app:C}.

    \item {\bf Experimental result reproducibility}
    \item[] Question: Does the paper fully disclose all the information needed to reproduce the main experimental results of the paper to the extent that it affects the main claims and/or conclusions of the paper (regardless of whether the code and data are provided or not)?
    \item[] Answer: \answerYes{} 
    \item[] Justification: Section 6 and Appendix.~\ref{app:experiments} specify the experimental setups.

\item {\bf Open access to data and code}
    \item[] Question: Does the paper provide open access to the data and code, with sufficient instructions to faithfully reproduce the main experimental results, as described in supplemental material?
    \item[] Answer: \answerYes{} 
    \item[] Justification: The code for these experiments is available at
\url{https://github.com/anonymousneurips2026/oracle-robust-online-alignment}.

\item {\bf Experimental setting/details}
    \item[] Question: Does the paper specify all the training and test details (e.g., data splits, hyperparameters, how they were chosen, type of optimizer) necessary to understand the results?
    \item[] Answer: \answerYes{} 
    \item[] Justification: Section~\ref{sec:experiments} and Appendix~\ref{app:experiments} specify the finite-response setups.

\item {\bf Experiment statistical significance}
    \item[] Question: Does the paper report error bars suitably and correctly defined or other appropriate information about the statistical significance of the experiments?
    \item[] Answer: \answerYes{} 
    \item[] Justification: See Section~\ref{sec:experiments}

\item {\bf Experiments compute resources}
    \item[] Question: For each experiment, does the paper provide sufficient information on the computer resources (type of compute workers, memory, time of execution) needed to reproduce the experiments?
    \item[] Answer: \answerYes{} 
    \item[] Justification: See appendix~\ref{app:finite-setup}
    
\item {\bf Code of ethics}
    \item[] Question: Does the research conducted in the paper conform, in every respect, with the NeurIPS Code of Ethics \url{https://neurips.cc/public/EthicsGuidelines}?
    \item[] Answer: \answerYes{} 
    \item[] Justification: The work is theoretical and uses controlled finite-response experiments with no human subjects, no personally identifiable data, and no release of high-risk models or datasets.

\item {\bf Broader impacts}
    \item[] Question: Does the paper discuss both potential positive societal impacts and negative societal impacts of the work performed?
    \item[] Answer: \answerYes{} 
    \item[] Justification: See appendix H
    
\item {\bf Safeguards}
    \item[] Question: Does the paper describe safeguards that have been put in place for responsible release of data or models that have a high risk for misuse (e.g., pre-trained language models, image generators, or scraped datasets)?
    \item[] Answer: \answerNA{} 
    \item[] Justification: The paper does not release pretrained language models, image generators, scraped datasets, or other assets with high misuse risk. 

\item {\bf Licenses for existing assets}
    \item[] Question: Are the creators or original owners of assets (e.g., code, data, models), used in the paper, properly credited and are the license and terms of use explicitly mentioned and properly respected?
    \item[] Answer: \answerYes{} 
    \item[] Justification: Appendix~\ref{app:experiments} discloses that the experiments use the frozen checkpoint
\texttt{Qwen/Qwen2.5-0.5B-Instruct} and its Apache-2.0 license is respected.

\item {\bf New assets}
    \item[] Question: Are new assets introduced in the paper well documented and is the documentation provided alongside the assets?
    \item[] Answer: \answerNA{} 
    \item[] Justification: No new assets are introduced in the paper.

\item {\bf Crowdsourcing and research with human subjects}
    \item[] Question: For crowdsourcing experiments and research with human subjects, does the paper include the full text of instructions given to participants and screenshots, if applicable, as well as details about compensation (if any)? 
    \item[] Answer: \answerNA{} 
    \item[] Justification: The paper does not involve crowdsourcing, human annotation, or human-subject data collection. 

\item {\bf Institutional review board (IRB) approvals or equivalent for research with human subjects}
    \item[] Question: Does the paper describe potential risks incurred by study participants, whether such risks were disclosed to the subjects, and whether Institutional Review Board (IRB) approvals (or an equivalent approval/review based on the requirements of your country or institution) were obtained?
    \item[] Answer: \answerNA{} 
    \item[] Justification: The paper does not involve crowdsourcing or human-subject research, so IRB approval or equivalent institutional review is not applicable.

\item {\bf Declaration of LLM usage}
    \item[] Question: Does the paper describe the usage of LLMs if it is an important, original, or non-standard component of the core methods in this research? Note that if the LLM is used only for writing, editing, or formatting purposes and does \emph{not} impact the core methodology, scientific rigor, or originality of the research, declaration is not required.
    \item[] Answer: \answerYes{} 
    \item[] Justification: Appendix~\ref{app:experiments} describes the limited LLM usage in the experiments:
\texttt{Qwen/Qwen2.5-0.5B-Instruct} is used as a frozen reference-logit provider for
constructing finite response sets.

\end{enumerate}